\documentclass[journal]{IEEEtran}

\usepackage[utf8]{inputenc}
\usepackage{color}
\usepackage{xcolor}
\usepackage{array}
\usepackage{verbatim}
\usepackage{float}
\usepackage{amsmath}
\usepackage{amsthm}
\usepackage{amssymb}
\usepackage{graphicx}
\usepackage{longtable}
\usepackage{multirow}
\usepackage{booktabs}
\usepackage{balance}

\usepackage[unicode=true,
bookmarks=false,
breaklinks=false,pdfborder={0 0 1},colorlinks=false]
{hyperref}
\hypersetup{
	colorlinks,bookmarksopen,bookmarksnumbered,citecolor=blue,urlcolor=blue}
\usepackage{cite}

\usepackage{lipsum}
\usepackage{mathtools}
\usepackage{cuted}

\usepackage{algorithmic}
\usepackage{longtable}

\floatstyle{ruled}
\newfloat{algorithm}{tbp}{loa}
\providecommand{\algorithmname}{Algorithm}
\floatname{algorithm}{\protect\algorithmname}

\makeatletter
\let\oldforeign@language\foreign@language
\DeclareRobustCommand{\foreign@language}[1]{%
	\lowercase{\oldforeign@language{#1}}}

\let\oldforeign@language\foreign@language
\DeclareRobustCommand{\foreign@language}[1]{%
	\lowercase{\oldforeign@language{#1}}}

\ifCLASSINFOpdf
\else
\fi

\newcommand{\MYfooter}{\smash{
		\hfil\parbox[t][\height][t]{\textwidth}{\centering
			\thepage}\hfil\hbox{}}}

\def\ps@IEEEtitlepagestyle{%
	\def\@oddhead{\parbox[t][\height][t]{\textwidth}{\centering \scriptsize
			Personal use of this material is permitted. Permission from the author(s) and/or copyright holder(s), must be obtained for all other uses. Please contact us and provide details if you believe this document breaches copyrights.\\
			\noindent\makebox[\linewidth]{}
		}\hfil\hbox{}}%
	\def\@evenhead{\scriptsize\thepage \hfil \leftmark\mbox{}}%
	\def\@oddfoot{\parbox[t][\height][l]{\textwidth}{
			\vspace{-20pt}{\rule{\textwidth}{0.4pt}}\\ \footnotesize\underline{To cite this article:}
			{\bf{\footnotesize\textcolor{red}{H. A. Hashim, A. E.E. Eltoukhy, and A. Odry "Observer-based Controller for VTOL-UAVs Tracking using Direct Vision-Aided Inertial Navigation Measurements," ISA Transactions, vol. 137, pp. 133-143, 2023.}}} doi: \href{https://doi.org/10.1016/j.isatra.2022.12.014}{10.1016/j.isatra.2022.12.014}\\
			\noindent\makebox[\linewidth]
		}\hfil\hbox{}}%
	\def\@evenfoot{\MYfooter}}

\makeatother
\newtheorem{defn}{Definition}

\newtheorem{lem}{Lemma}

\newtheorem{thm}{Theorem}
\newtheorem{rem}{Remark}

\newtheorem{assum}{Assumption}

\begin{document}
	\bstctlcite{IEEEexample:BSTcontrol}

\title{Observer-based Controller for VTOL-UAVs Tracking using Direct Vision-Aided Inertial Navigation Measurements}

\author{Hashim A. Hashim, Abdelrahman E.E. Eltoukhy, and Akos Odry
	\thanks{This work was supported in part by National Sciences and Engineering
		Research Council of Canada (NSERC), under the grants RGPIN-2022-04937, and in part by the Hong Kong Polytechnic University under grant P0036181 and RGC Hong Kong.}
	\thanks{H. A. Hashim is with the Department of Mechanical and Aerospace Engineering,
		Carleton University, Ottawa, Ontario, K1S-5B6, Canada, email: hhashim@carleton.ca}
	\thanks{A. E.E. Eltoukhy is with the Department of Industrial and Systems Engineering, The Hong Kong Polytechnic University, Hung Hum, 
	Hong Kong, e-mail: abdelrahman.eltoukhy@polyu.edu.hk}
	\thanks{A. Odry is with the Department of Mechatronics and Automation, Faculty of Engineering, University of Szeged, Moszkvai krt. 9, 6725 Szeged, Hungary}
}



\maketitle

\begin{abstract}
This paper proposes a novel observer-based controller for Vertical
Take-Off and Landing (VTOL) Unmanned Aerial Vehicle (UAV) designed
to directly receive measurements from a Vision-Aided Inertial Navigation
System (VA-INS) and produce the required thrust and rotational torque
inputs. The VA-INS is composed of a vision unit (monocular or stereo
camera) and a typical low-cost 6-axis Inertial Measurement Unit (IMU)
equipped with an accelerometer and a gyroscope. A major benefit of
this approach is its applicability for environments where the Global
Positioning System (GPS) is inaccessible. The proposed VTOL-UAV observer
utilizes IMU and feature measurements to accurately estimate attitude
(orientation), gyroscope bias, position, and linear velocity. Ability
to use VA-INS measurements directly makes the proposed observer design
more computationally efficient as it obviates the need for attitude
and position reconstruction. Once the motion components are estimated,
the observer-based controller is used to control the VTOL-UAV attitude,
angular velocity, position, and linear velocity guiding the vehicle
along the desired trajectory in six degrees of freedom (6 DoF). The
closed-loop estimation and the control errors of the observer-based
controller are proven to be exponentially stable starting from almost
any initial condition. To achieve global and unique VTOL-UAV representation
in 6 DoF, the proposed approach is posed on the Lie Group and the
design in unit-quaternion is presented. Although the proposed approach
is described in a continuous form, the discrete version is provided
and tested.
\end{abstract}


\IEEEpeerreviewmaketitle{}

\section{Introduction}

\IEEEPARstart{I}{n} the field of autonomous navigation, comprehensive autonomous modules
able to accurately estimate Unmanned Aerial Vehicle (UAV) motion components
and to provide control signals to successfully track the vehicle along
the desired trajectory are in great demand. When using a Vision-Aided
Inertial Navigation System (VA-INS) composed of a low-cost Inertial
Measurement Unit (IMU) and a vision unit (monocular or stereo camera),
the UAV motion components that require estimation will include orientation
(attitude), gyro bias, position, and linear velocity \cite{hashim2021ACC,hashim2021Navigation}.
Given the fact that rigid-body's attitude, position, and linear velocity
are generally unknown, they can be reconstructed utilizing sensor
measurements. UAV's orientation, commonly known as attitude, could
be obtained through a set of inertial-frame observations in addition
to the related body-frame measurements \cite{ivanov2018three,fei2021nano}.
However, it is crucial to note that widely available low-cost sensors
produce uncertain measurements leading to poor attitude determination
results. Therefore, alternative approaches, namely Gaussian filters
\cite{crassidis2007survey,chen2019higher,sabzevari2020symmetry} nonlinear
filters \cite{hashim2019SO3Wiley,grip2012attitude,hashim2018SO3Stochastic,zlotnik2016exponential},
and neuro-adaptive filters \cite{hashim2022adaptive,hashim2022NeuralFilter}
have been proposed providing better estimates in comparison to \cite{ivanov2018three,fei2021nano}.
A common choice of sensor for attitude estimation is a low-cost 6-axis
IMU composed of an accelerometer and a gyroscope. Moreover, an IMU
integrated with a vision unit can supply pose (\textit{i.e.}, attitude
and position) information of a vehicle navigating with six degrees
of freedom (6 DoF) \cite{hashim2020SE3Stochastic}. Widely used methods
of pose estimation include Gaussian filters \cite{filipe2015extended,wendel2006integrated}
and nonlinear filters \cite{hashim2020SE3Stochastic,baldwin2009nonlinear}.
Nonetheless, pose estimation solutions in \cite{filipe2015extended,wendel2006integrated,hashim2020SE3Stochastic,baldwin2009nonlinear}
produce good results only given the availability of linear velocity
measurements obtained, for instance, by the Global Positioning System
(GPS). In GPS-denied environments, however, rigid-body's linear velocity
is challenging to obtain \cite{hashim2021Navigation,hashim2022ExpVTOL}.
Also, the solutions in \cite{filipe2015extended,wendel2006integrated,hashim2020SE3Stochastic,baldwin2009nonlinear}
suppose that the accelerometer can supply the gravity measurements
in the vehicle's body-frame assuming negligible linear accelerations.
Several navigation solutions for acquiring rigid-body's attitude,
position, and linear velocity in absence of GPS have been proposed
that utilize IMU and feature measurements, namely indoor localization
\cite{xu2021maximum,song2020integrated}, a Kalman filter \cite{su2020variational},
an extended Kalman filter \cite{mourikis2009vision,zhong2020direct},
an unscented Kalman filter \cite{allotta2016unscented}, and a nonlinear
navigation filter based on VA-INS \cite{hashim2021ACC,hashim2021Navigation}.
All of the above-listed solutions address solely the estimation stage,
rather than proposing an estimator-based controller module.

Over the past few decades, UAVs, especially Vertical Take-Off and
Landing (VTOL)-UAVs, have become widely used stimulating interest
in and demand for UAV control solutions. Examples of such solutions
include backstepping control \cite{labbadi2019robust}, cascaded control
\cite{liao2016distributed,muslimov2020adaptive}, sliding mode control
\cite{zheng2014second,wu2019modeling}, hierarchical control \cite{liang2019novel},
formation control \cite{muslimov2020adaptive}, gain scheduling \cite{kaminer1998trajectory},
and prescribed performance \cite{koksal2020backstepping}, among others.
In spite of a significant effort made in \cite{labbadi2019robust,liao2016distributed,zheng2014second,wu2019modeling,liang2019novel,kaminer1998trajectory}
to control the rigid-body's pose, these solutions are heavily reliant
on precise knowledge of rigid-body's attitude, position, and angular
velocity. As previously stated, attitude and position can be obtained
using pose estimators known to produce good estimates, but yet these
estimates are not sufficiently accurate to ensure a safe control process.
In addition, angular velocity measurements supplied by a low-cost
IMU module are uncertain and are likely to be corrupted with unknown
bias \cite{hashim2021Navigation,hashim2022ExpVTOL}. As such, applying
the tracking control proposed in \cite{labbadi2019robust,liao2016distributed,zheng2014second,wu2019modeling,liang2019novel,kaminer1998trajectory}
to the information provided by a low-cost VA-INS module will in all
likelihood produce undesirable results potentially causing a UAV to
become unstable. Standalone observer and controller designs cannot
guarantee the interconnected observer-based controller module stability
\cite{hashim2022ExpVTOL}.

Controlling the UAV trajectory in presence of uncertain and bias-corrupted
measurements can be made possible by applying the image-based visual
servoing (IBVS) approach combined with an IMU module able to collect
the necessary UAV motion components, namely attitude, position, and
angular velocity \cite{hashim2021Navigation}. Examples of the above
approach include an autonomous landing of a VTOL-UAV using IBVS and
sliding mode control \cite{lee2012autonomous}, output-feedback control
for VTOL-UAV \cite{rafique2020output,li2021image}, cascaded control
for a quadrotor \cite{chen2019image}, model predictive control \cite{liu2021robust},
and a linear observer coupled with a translation and attitude controller
\cite{mokhtari2016extended}. The common feature of the above-mentioned
techniques \cite{lee2012autonomous,rafique2020output,li2021image,chen2019image,liu2021robust,mokhtari2016extended}
is a double loop structure reliant on Euler angles. The inner loop
controls the angular velocity using IMU data, whereas the outer stage
controls the thrust utilizing the vision measurements. The use of
Euler angles helps to visualize the three-dimensional orientation
of a rigid-body. However, the main weaknesses of the Euler angle representation
are the singularity at several configurations and inability to represent
the attitude globally \cite{hashim2019AtiitudeSurvey,shuster1993survey}.
Although unit-quaternion is not subject to singularity, it suffers
from non-uniqueness \cite{hashim2019AtiitudeSurvey,shuster1993survey}.
\textit{Lie Group}, in contrast, provides a nonsingular and unique
representation of the rigid-body's orientation \cite{grip2012attitude,hashim2019SO3Wiley}.
Considering the nonlinearity of the VTOL-UAV model dynamics and the
limitations of the Euler angles, the techniques in \cite{lee2012autonomous,rafique2020output,li2021image,chen2019image,liu2021robust,mokhtari2016extended}
are unable to guarantee global stability. A full-state observer-based
controller on the \textit{Lie Group} can bemployed for VTOL-UAVs to
resolve Euler angles singularities and non-uniqueness of unit-quaternion
\cite{hashim2022ExpVTOL}. However, the approaches in \cite{hashim2022ExpVTOL,lee2012autonomous,rafique2020output,chen2019image,liu2021robust,mokhtari2016extended}
require UAV pose reconstruction, which leads to a significant increase
in the computational cost. Moreover, attitude and position reconstructions
performed based on low-cost VA-INS measurements can be unreliable
due to high levels of uncertainties \cite{hashim2021Navigation,hashim2020SE3Stochastic}.
Considering the aforementioned limitations of the existing state-of-the-art
solutions, it becomes apparent that UAV observer-based controllers
able to use VA-INS measurements directly are in great demand. To this
end, the contributions of this paper are as follows:
\begin{itemize}
	\item A framework that allows for direct implementation of VA-INS measurements
	by a VTOL-UAV observer and controller has been established.
	\item A novel direct VA-INS-based nonlinear observer on the Lie Group that
	follows the true VTOL-UAV motion kinematics without the need for attitude
	and position reconstructions has been developed;
	\item The proposed observer successfully estimates the VTOL-UAV's attitude,
	gyro-bias (IMU uncertain measurements), position, and linear velocity
	guaranteeing almost global exponential stability of the error signals
	starting from almost any initial condition; and
	\item Novel control laws based on and tightly-coupled with the proposed
	VA-INS-based nonlinear observer, are developed to control the UAV's
	attitude, position, angular velocity, and linear velocity guaranteeing
	almost global exponential stability of the closed-loop error signals
	starting from almost any initial condition.
\end{itemize}
To the best of the authors knowledge, observer-based controllers for
VTOL-UAVs able to directly use VA-INS measurements without attitude
and position reconstructions guaranteeing almost global exponential
stability remain an unaddressed challenge.

This paper is composed of seven Sections. Section \ref{sec:Preliminaries-and-Math}
gives an overview of \textit{Lie Group}, math notation, identities,
establishes the problem, and introduces the available VA-INS measurements.
Section \ref{sec:VTOL_Observer} presents a direct nonlinear observer
for a VTOL-UAV in a continuous form. Section \ref{sec:VTOL_Controller}
introduces a novel control strategy for VTOL-UAV in a continuous form.
Section \ref{sec:VTOL_Implementation} presents the implementation
details in a discrete form. Section \ref{sec:SE3_Simulations} depicts
the effectiveness of the proposed methodology. Lastly, Section \ref{sec:SE3_Conclusion}
summarizes the work.

\section{Preliminaries and Problem Formulation\label{sec:Preliminaries-and-Math}}

\subsection{Preliminaries and Notation}

In this paper, the set of real numbers is described by $\mathbb{R}$,
an $a$-by-$b$ real dimensional space is represented by $\mathbb{R}^{a\times b}$,
and non-negative real numbers are denoted by $\mathbb{R}_{+}$. $||y||=\sqrt{y^{\top}y}$
refers to Euclidean norm of a vector $y\in\mathbb{R}^{a}$. $\mathbf{I}_{a}$
refers to an $a$-by-$a$ identity matrix, and $0_{a\times b}$ denotes
an $a$-by-$b$ zero matrix. $\lambda(W)=\{\lambda_{1},\lambda_{2},\ldots,\lambda_{a}\}$
stands for the set of eigenvalues of $W\in\mathbb{R}^{a\times a}$
with $\overline{\lambda}_{W}=\overline{\lambda}(W)$ being the set's
maximum value and $\underline{\lambda}_{W}=\underline{\lambda}(W)$
being the set's minimum value. Consider a UAV-VTOL traveling in 6
DoF where 
\begin{itemize}
	\item $\left\{ \mathcal{B}\right\} =\{e_{\mathcal{B}1},e_{\mathcal{B}2},e_{\mathcal{B}3}\}$
	denotes the vehicle's body-frame and
	\item $\left\{ \mathcal{I}\right\} =\{e_{1},e_{2},e_{3}\}$ refers to a
	fixed inertial-frame such that $e_{1}:=[1,0,0]^{\top}$, $e_{2}:=[0,1,0]^{\top}$,
	and $e_{3}:=[0,0,1]^{\top}$.
\end{itemize}
$\mathbb{SO}\left(3\right)$ denotes \textit{Special Orthogonal Group}
defined by 
\[
\mathbb{SO}\left(3\right)=\left\{ \left.R\in\mathbb{R}^{3\times3}\right|RR^{\top}=R^{\top}R=\mathbf{I}_{3}\text{, }{\rm det}\left(R\right)=+1\right\} 
\]
where $R\in\mathbb{SO}\left(3\right)$ is vehicle's orientation,
also known as attitude. $\mathfrak{so}\left(3\right)$ describes the
\textit{Lie-algebra} of $\mathbb{SO}\left(3\right)$ such that
\begin{align*}
	\mathfrak{so}\left(3\right) & =\left\{ \left.\left[\Omega\right]_{\times}\in\mathbb{R}^{3\times3}\right|\left[\Omega\right]_{\times}^{\top}=-\left[\Omega\right]_{\times}\right\} 
\end{align*}
where
\[
\left[\Omega\right]_{\times}=\left[\begin{array}{ccc}
	0 & -\Omega_{3} & \Omega_{2}\\
	\Omega_{3} & 0 & -\Omega_{1}\\
	-\Omega_{2} & \Omega_{1} & 0
\end{array}\right]\in\mathfrak{so}\left(3\right),\hspace{1em}\Omega=\left[\begin{array}{c}
	\Omega_{1}\\
	\Omega_{2}\\
	\Omega_{3}
\end{array}\right]
\]
$\mathbf{vex}:\mathfrak{so}\left(3\right)\rightarrow\mathbb{R}^{3}$
defines the map of $\left[\cdot\right]_{\times}$ to $\mathbb{R}^{3}$
where $\mathbf{vex}([\Omega]_{\times})=\Omega$ for all $\Omega\in\mathbb{R}^{3}$.
$\boldsymbol{\mathcal{P}}_{a}:\mathbb{R}^{3\times3}\rightarrow\mathfrak{so}\left(3\right)$
is the anti-symmetric projection operator, while $\boldsymbol{\Upsilon}(\cdot)$
is a composition mapping of $\mathbf{vex}\circ\boldsymbol{\mathcal{P}}_{a}$
given by
\begin{align}
	\boldsymbol{\mathcal{P}}_{a}(M) & =\frac{1}{2}(M-M^{\top})\in\mathfrak{so}\left(3\right),\hspace{1em}\forall M\in\mathbb{R}^{3\times3}\label{eq:VTOL_Pa}\\
	\boldsymbol{\Upsilon}(M) & =\mathbf{vex}(\boldsymbol{\mathcal{P}}_{a}(M))\in\mathbb{R}^{3},\hspace{1em}\forall M\in\mathbb{R}^{3\times3}\label{eq:VTOL_VEX}
\end{align}
Define the following map:
\begin{equation}
	||R||_{{\rm I}}=\frac{1}{4}{\rm Tr}\{\mathbf{I}_{3}-R\}\in[0,1],\hspace{1em}R\in\mathbb{SO}\left(3\right)\label{eq:VTOL_Ecul_Dist}
\end{equation}
Define $\mathbb{SE}_{2}\left(3\right)=\mathbb{SO}\left(3\right)\times\mathbb{R}^{3}\times\mathbb{R}^{3}\subset\mathbb{R}^{5\times5}$
\cite{barrau2016invariant} as the extended \textit{Special Euclidean
	Group} where
\begin{align}
	\mathbb{SE}_{2}\left(3\right) & =\{\left.X\in\mathbb{R}^{5\times5}\right|R\in\mathbb{SO}\left(3\right),P,V\in\mathbb{R}^{3}\}\label{eq:VTOL_SE2_3}\\
	X=f( & R^{\top},P,V)=\left[\begin{array}{ccc}
		R^{\top} & P & V\\
		0_{1\times3} & 1 & 0\\
		0_{1\times3} & 0 & 1
	\end{array}\right]\in\mathbb{SE}_{2}\left(3\right)\label{eq:VTOL_X}
\end{align}
such that $X$ is the homogeneous navigation matrix and $R$, $P$,
and $V$ describe the vehicle's attitude, position, and linear velocity,
respectively \cite{hashim2021ACC,hashim2021Navigation,hashim2022ExpVTOL,barrau2016invariant}.
One has 
\[
X^{-1}=\left[\begin{array}{ccc}
	R & -RP & -RV\\
	0_{1\times3} & 1 & 0\\
	0_{1\times3} & 0 & 1
\end{array}\right]\in\mathbb{SE}_{2}\left(3\right)
\]
Let $\mathcal{U}_{\mathcal{M}}=\mathfrak{so}\left(3\right)\times\mathbb{R}^{3}\times\mathbb{R}^{3}\times\mathbb{R}\subset\mathbb{R}^{5\times5}$
be defined as
\begin{align}
	\mathcal{U}_{\mathcal{M}} & =\left\{ \left.u(\text{[\ensuremath{\Omega\text{\ensuremath{]_{\times}}}}},V,e_{3},\kappa)\right|[\Omega\text{\ensuremath{]_{\times}}}\in\mathfrak{so}\left(3\right),V,e_{3}\in\mathbb{R}^{3},\kappa\in\mathbb{R}\right\} \nonumber \\
	& U=u([\Omega\text{\ensuremath{]_{\times}}},V,e_{3},\kappa)=\left[\begin{array}{ccc}
		[\Omega\text{\ensuremath{]_{\times}}} & V & e_{3}\\
		0_{1\times3} & 0 & 0\\
		0_{1\times3} & \kappa & 0
	\end{array}\right]\in\mathcal{U}_{\mathcal{M}}\label{eq:VTOL_Mu}
\end{align}
To learn more about the group $\mathbb{SE}_{2}\left(3\right)$, the
navigation matrix $X\in\mathbb{SE}_{2}\left(3\right)$, and $\mathcal{U}_{\mathcal{M}}$
consult \cite{hashim2021ACC,hashim2021Navigation,hashim2022ExpVTOL}.
The following identity is used throughout this paper:
\begin{align}
	{\rm Tr}\{N[\Omega]_{\times}\}= & {\rm Tr}\{\boldsymbol{\mathcal{P}}_{a}(N)[\Omega]_{\times}\},\hspace{1em}\Omega\in{\rm \mathbb{R}}^{3},N\in\mathbb{R}^{3\times3}\nonumber \\
	= & -2\mathbf{vex}(\boldsymbol{\mathcal{P}}_{a}(N))^{\top}\Omega\label{eq:VTOL_R_Identity2}\\{}
	[y\times z]_{\times}= & zy^{\top}-yz^{\top},\hspace{1em}y,z\in{\rm \mathbb{R}}^{3}\label{eq:VTOL_R_Identity3}
\end{align}

\subsection{Problem Formulation\label{sec:SE3_Problem-Formulation}}

Let $R\in\mathbb{SO}\left(3\right)$, $P\in\mathbb{R}^{3}$, and $V\in\mathbb{R}^{3}$
be the true attitude, position, and linear velocity of a VTOL-UAV
traveling in 6 DoF, respectively. $R$, $P$, and $V$ are assumed
to be completely unknown.%
{} The VTOL-UAV motion equations are as follows \cite{hashim2022ExpVTOL}:
\begin{align}
	& \begin{cases}
		\dot{R} & =-\left[\Omega\right]_{\times}R\\
		J\dot{\Omega} & =\left[J\Omega\right]_{\times}\Omega+\mathcal{T}
	\end{cases},\hspace{1em}R,\Omega,\mathcal{T}\in\{\mathcal{B}\}\label{eq:VTOL_Rotation}\\
	& \begin{cases}
		\dot{P} & =V\\
		\dot{V} & =ge_{3}-\frac{\Im}{m}R^{\top}e_{3}
	\end{cases},\hspace{1em}P,V\in\{\mathcal{I}\}\label{eq:VTOL_Translation}
\end{align}
where $\Omega\in\mathbb{R}^{3}$, $J\in\mathbb{R}^{3\times3}$, $m\in\mathbb{R}$,
$g\in\mathbb{R}$, and $e_{3}=[0,0,1]^{\top}$ describe angular velocity,
a constant symmetric positive definite inertia matrix, UAV's mass,
gravitational acceleration, and a basis vector, respectively. Also,
$\mathcal{T}\in\mathbb{R}^{3}$ and $\Im\in\mathbb{R}$ refer to rotational
torque input and thrust input, respectively. In view of \eqref{eq:VTOL_Rotation}
and \eqref{eq:VTOL_Translation}, the VTOL-UAV motion nonlinear equations
can be re-expressed as
\begin{equation}
	\begin{cases}
		\dot{X} & =XU-\mathcal{G}X\\
		J\dot{\Omega} & =\left[J\Omega\right]_{\times}\Omega+\mathcal{T}
	\end{cases},\hspace{1em}U,\mathcal{G}\in\mathcal{U}_{\mathcal{M}},\,X\in\mathbb{SE}_{2}\left(3\right)\label{eq:VTOL_Compact}
\end{equation}
where, based on the definition in \eqref{eq:VTOL_X} and \eqref{eq:VTOL_Mu},
$X$ is the navigation matrix, 
\[
U=u([\Omega\text{\ensuremath{]_{\times}}},0_{3\times1},-\frac{\Im}{m}e_{3},1)=\left[\begin{array}{ccc}
	[\Omega\text{\ensuremath{]_{\times}}} & 0_{3\times1} & -\frac{\Im}{m}e_{3}\\
	0_{1\times3} & 0 & 0\\
	0_{1\times3} & 1 & 0
\end{array}\right]
\]
and $\mathcal{G}=u(0_{3\times3},0_{3\times1},-ge_{3},1)$. Therefore,
it can be shown that the dynamics in \eqref{eq:VTOL_Compact} have
the map $\mathbb{SE}_{2}\left(3\right)\times\mathcal{U}_{\mathcal{M}}\rightarrow T_{X}\mathbb{SE}_{2}\left(3\right)\in\mathbb{R}^{5\times5}$
where $\dot{X}\in T_{X}\mathbb{SE}_{2}\left(3\right)$. Attitude and
position of a UAV can be extracted through a group of observations
in $\{\mathcal{I}\}$ and the corresponding measurements in $\{\mathcal{B}\}$
\cite{hashim2020SE3Stochastic}. Let $p_{i}$ denote the $i$th known
inertial-frame feature observation. The respective $i$th body-frame
feature measurement is given by \cite{hashim2021ACC,hashim2020SE3Stochastic}
\begin{align}
	y_{i} & =R(p_{i}-P)+b_{i}\in\mathbb{R}^{3},\hspace{1em}\forall i=1,\ldots,n\label{eq:VTOL_VRP}
\end{align}
where $b_{i}$ describes an unknown uncertainty.

\begin{assum}\label{Assum:VTOL_1Landmark}(UAV attitude and position
	observability) The attitude $R\in\mathbb{SO}\left(3\right)$ and position
	$P\in\mathbb{R}^{3}$ of a UAV can be defined if there are three or
	more non-collinear observations and their measurements defined in
	\eqref{eq:VTOL_VRP}.\end{assum}

Assumption \ref{Assum:VTOL_1Landmark} is common for problems involving
attitude and position estimation \cite{baldwin2009nonlinear,hashim2020SE3Stochastic}.
\begin{lem}
	\label{Lemm:vex_RI}\cite{hashim2019SO3Wiley} Let $R\in\mathbb{SO}\left(3\right)$,
	$M=M^{\top}\in\mathbb{R}^{3\times3}$ where $rank(M)\geq2$, and $\overline{M}={\rm Tr}\{M\}\mathbf{I}_{3}-M$.
	Thereby, the following two definitions hold:
	\begin{align}
		||\boldsymbol{\Upsilon}(R)||^{2} & =4(1-||R||_{{\rm I}})||R||_{{\rm I}}\label{eq:VTOL_lemm_vexRI}\\
		\underline{\lambda}_{\overline{M}}^{2}||R||_{{\rm I}} & \leq||\boldsymbol{\Upsilon}(RM)||^{2}\leq\overline{\lambda}_{\overline{M}}^{2}||R||_{{\rm I}}\label{eq:VTOL_lemm_vexRIM}
	\end{align}
\end{lem}
\begin{defn}
	\label{def:Unstable-set}\cite{hashim2019SO3Wiley} A forward invariant
	non-attractive set $\mathcal{S}_{u}\subseteq\mathbb{SO}\left(3\right)$
	is defined by
	\begin{equation}
		\mathcal{S}_{u}=\{\left.R(0)\in\mathbb{SO}\left(3\right)\right|{\rm Tr}\{R(0)\}=-1\}\label{eq:SO3_PPF_STCH_SET}
	\end{equation}
	with $R(0)\in\mathcal{S}_{u}$ if one of the following three conditions
	holds: $R(0)={\rm diag}(1,-1,-1)$, $R(0)={\rm diag}(-1,1,-1)$, or
	$R(0)={\rm diag}(-1,-1,1)$.
\end{defn}
\begin{lem}
	\label{lem:Barbalat}(Barbalat Lemma Extension) Let $x(t)$ be a solution
	of a differential equation $\dot{x}(t)=f(t)+g(t)$ where $f(t)$ is
	a uniformly continuous function. Suppose that $\lim_{t\rightarrow\infty}x(t)=k_{c}$
	and $\lim_{t\rightarrow\infty}g(t)=0$, where $k_{c}$ denotes a constant.
	Then, $\lim_{t\rightarrow\infty}\dot{x}(t)=0$.
\end{lem}

\section{Direct Observer Design on Lie Group \label{sec:VTOL_Observer}}

This Section aims to design a nonlinear observer for a VTOL-UAV on
the \textit{Lie Group} that can be implemented using a group of measurements
directly alleviating the neccessity for attitude and position reconstruction.
Define an angular velocity measurement obtained by a gyroscope as
\begin{equation}
	\Omega_{m}=\Omega+b_{\Omega}\in\mathbb{R}^{3}\label{eq:eq:VTOL_Om_m}
\end{equation}
where $b_{\Omega}$ denotes unknown uncertainty (bias) of the gyro
measurement (for more information abiut IMU visit \cite{hashim2019SO3Wiley,grip2012attitude,hashim2018SO3Stochastic,zlotnik2016exponential}).
Let $b_{\Omega}$ be a positive constant, and let 
\[
\hat{X}=f(\hat{R}^{\top},\hat{P},\hat{V})=\left[\begin{array}{ccc}
	\hat{R}^{\top} & \hat{P} & \hat{V}\\
	0_{1\times3} & 1 & 0\\
	0_{1\times3} & 0 & 1
\end{array}\right]\in\mathbb{SE}_{2}\left(3\right)
\]
where $\hat{R}\in\mathbb{SO}\left(3\right)$, $\hat{b}_{\Omega}\in\mathbb{R}^{3}$,
$\hat{P}\in\mathbb{R}^{3}$, and $\hat{V}\in\mathbb{R}^{3}$ describe
estimates of navigation matrix, attitude, gyro bias, position, and
linear velocity, respectively. Define the error between the true and
the estimated values as follows:
\begin{align}
	\tilde{R}_{o}= & \hat{R}^{\top}R\label{eq:VTOL_Rerr}\\
	\tilde{b}_{\Omega}= & b_{\Omega}-\hat{b}_{\Omega}\label{eq:VTOL_bOmerr}\\
	\tilde{P}_{o}= & \hat{P}-\tilde{R}_{o}P\label{eq:VTOL_Perr}\\
	\tilde{V}_{o}= & \hat{V}-\tilde{R}_{o}V\label{eq:VTOL_Verr}
\end{align}
where $\tilde{R}_{o}$ is an attitude error, $\tilde{b}_{\Omega}$
is a gyroscope error, $\tilde{P}_{o}$ is a position error, and $\tilde{V}_{o}$
is a linear velocity error for all $\tilde{R}_{o}\in\mathbb{SO}\left(3\right)$
and $\tilde{b}_{\Omega},\tilde{P}_{o},\tilde{V}_{o}\in\mathbb{R}^{3}$.
According to \eqref{eq:VTOL_Rerr}, \eqref{eq:VTOL_Perr}, and \eqref{eq:VTOL_Verr},
one finds that $\tilde{X}_{o}=\hat{X}X^{-1}=f(\tilde{R}_{o},\tilde{P}_{o},\tilde{V}_{o})\in\mathbb{SE}_{2}\left(3\right)$.

\subsection{Measurement Set-up}

Let $s_{i}$ denote the confidence level of the $i$th sensor measurement,
and define $s_{T}=\sum_{i=1}^{n}s_{i}$ such that Assumption \ref{Assum:VTOL_1Landmark}
is met ($n\geq3$). Let $y_{i}=R(p_{i}-P)$ and $p_{c}=\frac{1}{s_{T}}\sum_{i=1}^{n}s_{i}p_{i}$,
and define $M=\sum_{i=1}^{n}s_{i}(p_{i}-p_{c})(p_{i}-p_{c})^{\top}$
which can be easily transformed into $M=\sum_{i=1}^{n}s_{i}p_{i}p_{i}^{\top}-s_{T}p_{c}p_{c}^{\top}$.
Therefore,
\begin{align}
	\sum_{i=1}^{n}s_{i}\hat{R}^{\top}y_{i} & (p_{i}-p_{c})^{\top}=\sum_{i=1}^{n}s_{i}\hat{R}^{\top}R(p_{i}-P)(p_{i}-p_{c})^{\top}\nonumber \\
	= & \tilde{R}_{o}\sum_{i=1}^{n}s_{i}(p_{i}p_{i}^{\top}-Pp_{i}^{\top}-p_{i}p_{c}^{\top}+Pp_{c}^{\top})\nonumber \\
	= & \tilde{R}_{o}(\sum_{i=1}^{n}s_{i}p_{i}p_{i}^{\top}-s_{T}p_{c}p_{c}^{\top})=\tilde{R}_{o}M\label{eq:VTOL_RM}
\end{align}
Therefore, one shows
\begin{align}
	& \boldsymbol{\mathcal{P}}_{a}(\tilde{R}_{o}M)=\boldsymbol{\mathcal{P}}_{a}(\sum_{i=1}^{n}s_{i}\hat{R}^{\top}y_{i}(p_{i}-p_{c})^{\top})\nonumber \\
	& \hspace{1.5em}=\sum_{i=1}^{n}\frac{s_{i}}{2}\left(\hat{R}^{\top}y_{i}(p_{i}-p_{c})^{\top}-(p_{i}-p_{c})y_{i}^{\top}\hat{R}\right)\label{eq:VTOL_Pa_RM}
\end{align}
Based on \eqref{eq:VTOL_VEX}, $\boldsymbol{\Upsilon}(\tilde{R}_{o}M)=\mathbf{vex}(\boldsymbol{\mathcal{P}}_{a}(\tilde{R}_{o}M))$.
Consequently, one has
\begin{align}
	\boldsymbol{\Upsilon}(\tilde{R}_{o}M) & =\sum_{i=1}^{n}\frac{s_{i}}{2}\left((p_{i}-p_{c})\times\hat{R}^{\top}y_{i}\right)\label{eq:VTOL_VEX_RM}
\end{align}
Also, one finds
\begin{align}
	\sum_{i=1}^{n}s_{i}\tilde{y}_{i} & =\sum_{i=1}^{n}s_{i}(\hat{P}+\hat{R}^{\top}y_{i}-p_{i})\nonumber \\
	& =\sum_{i=1}^{n}s_{i}\tilde{P}_{o}+\sum_{i=1}^{n}s_{i}(\tilde{R}_{o}p_{i}-p_{i})\nonumber \\
	& =s_{T}\tilde{P}_{o}+s_{T}(\tilde{R}_{o}-\mathbf{I}_{3})p_{c}\label{eq:VTOL_ye}
\end{align}
Based on \eqref{eq:VTOL_R_Identity3} and \eqref{eq:VTOL_ye}, if
$\sum_{i=1}^{n}s_{i}\tilde{y}_{i}\rightarrow0$ and $\tilde{R}_{o}\rightarrow\mathbf{I}_{3}$,
one has $\tilde{P}_{o}\rightarrow0$. Let us summarize all the aforementioned
measurements as follows:
\begin{equation}
	\begin{cases}
		p_{c} & =\frac{1}{s_{T}}\sum_{i=1}^{n}s_{i}p_{i},\hspace{1em}s_{T}=\sum_{i=1}^{n}s_{i}\\
		M & =\sum_{i=1}^{n}s_{i}p_{i}p_{i}^{\top}-s_{T}p_{c}p_{c}^{\top}\\
		\tilde{R}_{o}M & =\sum_{i=1}^{n}s_{i}\hat{R}^{\top}y_{i}(p_{i}-p_{c})^{\top}\\
		\boldsymbol{\Upsilon}(\tilde{R}_{o}M) & =\sum_{i=1}^{n}\frac{s_{i}}{2}\left((p_{i}-p_{c})\times\hat{R}^{\top}y_{i}\right)\\
		\sum_{i=1}^{n}s_{i}\tilde{y}_{i} & =\sum_{i=1}^{n}s_{i}(\hat{P}+\hat{R}^{\top}y_{i}-p_{i})\\
		& =s_{T}\tilde{P}_{o}+s_{T}(\tilde{R}_{o}-\mathbf{I}_{3})p_{c}
	\end{cases}\label{eq:VTOL_Measurements}
\end{equation}

\subsection{Direct Observer Design }

The objective of the designed observer is to drive $\hat{R}\rightarrow R$,
$\hat{b}_{\Omega}\rightarrow b_{\Omega}$, $\hat{P}\rightarrow P$,
and $\hat{V}\rightarrow V$ using direct measurements such that $\lim_{t\rightarrow\infty}\tilde{R}_{o}=\mathbf{I}_{3}$
and $\lim_{t\rightarrow\infty}\tilde{b}_{\Omega}=\lim_{t\rightarrow\infty}\tilde{P}_{o}=\lim_{t\rightarrow\infty}\tilde{V}_{o}=0_{3\times1}$.
Consider $\boldsymbol{\Upsilon}(\tilde{R}_{o}M)$ and $\sum_{i=1}^{n}s_{i}\tilde{y}_{i}$
are defined with respect to the direct measurements in \eqref{eq:VTOL_Measurements}:
\begin{align*}
	\boldsymbol{\Upsilon}(\tilde{R}_{o}M)= & \sum_{i=1}^{n}\frac{s_{i}}{2}\left((p_{i}-p_{c})\times\hat{R}^{\top}y_{i}\right)\\
	\sum_{i=1}^{n}s_{i}\tilde{y}_{i}= & \sum_{i=1}^{n}s_{i}(\hat{P}+\hat{R}^{\top}y_{i}-p_{i})
\end{align*}
Consider the following compact form of the novel direct nonlinear
observer for VTOL-UAV on the \textit{Lie Group} $\dot{\hat{X}}=\mathbb{SE}_{2}\left(3\right)\times\mathcal{U}_{\mathcal{M}}\rightarrow T_{\hat{X}}\mathbb{SE}_{2}\left(3\right)$:
\begin{equation}
	\begin{cases}
		\dot{\hat{X}} & =\hat{X}\hat{U}-W\hat{X},\hspace{1em}\hat{U},W\in\mathcal{U}_{\mathcal{M}},\,\hat{X}\in\mathbb{SE}_{2}\left(3\right)\\
		\dot{\hat{b}}_{\Omega} & =\gamma_{o}\hat{R}\boldsymbol{\Upsilon}(\tilde{R}_{o}M)
	\end{cases}\label{eq:VTOL_ObsvCompact}
\end{equation}
where $\hat{X}=f(\hat{R}^{\top},\hat{P},\hat{V})$ denotes the navigation
matrix estimate, $\hat{U}=u([\Omega_{m}-\hat{b}_{\Omega}\text{\ensuremath{]_{\times}}},0_{3\times1},-\frac{\Im}{m}e_{3},1)$,
and $W=u([w_{\Omega}]_{\times},w_{V},w_{a},1)$, see \eqref{eq:VTOL_X}
and \eqref{eq:VTOL_Mu}. $J$, $g$, and $m$ stand for vehicle's
inertia matrix, gravitational acceleration, and mass, respectively.
It becomes apparent that $\dot{\hat{X}}$ in \eqref{eq:VTOL_ObsvCompact}
mimics the true VTOL-UAV motion dynamics. Also, the correction factors
in $W$ are as follows: 
\begin{equation}
	\begin{cases}
		w_{\Omega} & =k_{o1}\boldsymbol{\Upsilon}(\tilde{R}_{o}M)\\
		w_{V} & =k_{o2}\sum_{i=1}^{n}s_{i}\tilde{y}_{i}-\frac{1}{s_{T}}[w_{\Omega}]_{\times}(\sum_{i=1}^{n}s_{i}\tilde{y}_{i}+s_{T}p_{c})\\
		w_{a} & =-ge_{3}+k_{o3}\sum_{i=1}^{n}s_{i}\tilde{y}_{i}
	\end{cases}\label{eq:VTOL_Wcorr}
\end{equation}
Furthermore, $\gamma_{o}$, $k_{o1}$, $k_{o2}$, and $k_{o3}$ are
positive constants. The nonlinear estimator in \eqref{eq:VTOL_ObsvCompact}
can be detailed as follows:
\begin{align}
	\dot{\hat{R}} & =-[\Omega_{m}-\hat{b}_{\Omega}]_{\times}\hat{R}+\hat{R}[w_{\Omega}]_{\times}\label{eq:VTOL_Restdot}\\
	\dot{\hat{b}}_{\Omega} & =\gamma_{o}\hat{R}\boldsymbol{\Upsilon}(\tilde{R}_{o}M)\label{eq:VTOL_Omestdot}\\
	\dot{\hat{P}} & =\hat{V}-[w_{\Omega}]_{\times}\hat{P}-w_{V}\label{eq:VTOL_Pestdot}\\
	\dot{\hat{V}} & =-\frac{\Im}{m}\hat{R}^{\top}e_{3}-[w_{\Omega}]_{\times}\hat{V}-w_{a}\label{eq:VTOL_Vestdot}
\end{align}

\begin{thm}
	\label{thm:Theorem1}Consider the true VTOL-UAV motion dynamics in
	\eqref{eq:VTOL_Rotation} and \eqref{eq:VTOL_Translation}, and let
	Assumption \ref{Assum:VTOL_1Landmark} hold true where $y_{i}=R(p_{i}-P)$
	for all $i=1,2,\ldots,n$. Let the observer in \eqref{eq:VTOL_ObsvCompact}
	and the correction factors in \eqref{eq:VTOL_Wcorr} be coupled with
	the measurements in \eqref{eq:VTOL_Measurements}. Then for $\tilde{R}_{o}(0)\notin\mathcal{S}_{u}$,
	$\lim_{t\rightarrow\infty}\tilde{R}_{o}=\mathbf{I}_{3}$, $\lim_{t\rightarrow\infty}\tilde{b}_{\Omega}=\lim_{t\rightarrow\infty}\tilde{P}_{o}=\lim_{t\rightarrow\infty}\tilde{V}_{o}=0_{3\times1}$,
	and the closed-loop error signals are uniformly almost globally exponentially
	stable.
\end{thm}
\begin{proof}From \eqref{eq:VTOL_Rerr}, \eqref{eq:VTOL_Rotation},
	and \eqref{eq:VTOL_Restdot}, one finds
	\begin{align}
		\dot{\tilde{R}}_{o} & =\dot{\hat{R}}^{\top}R+\hat{R}^{\top}\dot{R}\nonumber \\
		& =[\hat{R}^{\top}(\Omega_{m}-\hat{b}_{\Omega})-w_{\Omega}]_{\times}\tilde{R}_{o}-[\hat{R}^{\top}\Omega]_{\times}\tilde{R}_{o}\nonumber \\
		& =-[w_{\Omega}-\hat{R}^{\top}\tilde{b}_{\Omega}]_{\times}\tilde{R}_{o}\label{eq:VTOL_Rerr_dot}
	\end{align}
	Define $||\tilde{R}_{o}||_{{\rm I}}=\frac{1}{4}{\rm Tr}\{\mathbf{I}_{3}-\tilde{R}_{o}\}$
	and $||\tilde{R}_{o}M||_{{\rm I}}=\frac{1}{4}{\rm Tr}\{(\mathbf{I}_{3}-\tilde{R}_{o})M\}$.
	From \eqref{eq:VTOL_Ecul_Dist}, \eqref{eq:VTOL_R_Identity2}, and
	\eqref{eq:VTOL_Rerr_dot}, one finds \cite{hashim2019AtiitudeSurvey}
	\begin{align}
		\frac{d}{dt}||\tilde{R}_{o}||_{{\rm I}}= & -\frac{1}{2}\boldsymbol{\Upsilon}(\tilde{R}_{o})^{\top}(w_{\Omega}-\hat{R}^{\top}\tilde{b}_{\Omega})\label{eq:VTOL_RerrI_dot}\\
		\frac{d}{dt}||\tilde{R}_{o}M||_{{\rm I}}= & -\frac{1}{2}\boldsymbol{\Upsilon}(\tilde{R}_{o}M)^{\top}(w_{\Omega}-\hat{R}^{\top}\tilde{b}_{\Omega})\label{eq:VTOL_RerrMI_dot}
	\end{align}
	Using \eqref{eq:VTOL_Translation}, \eqref{eq:VTOL_Perr}, and \eqref{eq:VTOL_Pestdot},
	one realizes that
	\begin{align*}
		\dot{\tilde{P}}_{o}= & \tilde{V}_{o}-\frac{1}{s_{T}}[w_{\Omega}]_{\times}\sum_{i=1}^{n}s_{i}\tilde{y}_{i}+[w_{\Omega}]_{\times}(\tilde{R}_{o}-\mathbf{I}_{3})p_{c}\\
		& -w_{V}-[\hat{R}^{\top}\tilde{b}_{\Omega}]_{\times}\tilde{R}_{o}P
	\end{align*}
	Using $\dot{\tilde{P}}_{o}$, $w_{V}$ in \eqref{eq:VTOL_Wcorr},
	and \eqref{eq:VTOL_ye}, one has
	\begin{align}
		\sum_{i=1}^{n}s_{i}\dot{\tilde{y}}_{i} & =s_{T}\tilde{V}_{o}-s_{T}k_{o2}\sum_{i=1}^{n}s_{i}\tilde{y}_{i}-s_{T}[\hat{R}^{\top}\tilde{b}_{\Omega}]_{\times}\tilde{R}_{o}(P-p_{c})\label{eq:VTOL_yerr_dot}
	\end{align}
	From \eqref{eq:VTOL_Translation}, \eqref{eq:VTOL_Verr}, \eqref{eq:VTOL_Vestdot},
	and using $w_{a}$ in \eqref{eq:VTOL_Wcorr}, one finds that
	\begin{align}
		\dot{\tilde{V}}_{o}= & -[w_{\Omega}]_{\times}\tilde{V}_{o}-k_{o3}\sum_{i=1}^{n}s_{i}\tilde{y}_{i}-[\hat{R}^{\top}\tilde{b}_{\Omega}]_{\times}\tilde{R}_{o}V\nonumber \\
		& +g(\mathbf{I}_{3}-\tilde{R})e_{3}\label{eq:VTOL_Verr_dot}
	\end{align}
	From \eqref{eq:VTOL_Wcorr} and \eqref{eq:VTOL_Rerr_dot}, one is
	able to show that $\boldsymbol{\Upsilon}(\dot{\tilde{R}}_{o})=-\frac{1}{2}\Psi(\tilde{R}_{o})(w_{\Omega}-\hat{R}^{\top}\tilde{b}_{\Omega})$
	such that \cite{hashim2019AtiitudeSurvey}
	\begin{align}
		\boldsymbol{\Upsilon}(\dot{\tilde{R}}_{o}) & =-\frac{1}{2}\Psi(\tilde{R}_{o})(k_{o1}\boldsymbol{\Upsilon}(\tilde{R}_{o}M)-\hat{R}^{\top}\tilde{b}_{\Omega})\label{eq:VTOL_vex_dot}
	\end{align}
	where $\Psi(\tilde{R}_{o})={\rm Tr}\{\tilde{R}_{o}\}\mathbf{I}_{3}-\tilde{R}_{o}$.
	In view of Lemma \ref{Lemm:vex_RI}, \eqref{eq:VTOL_vex_dot}, and
	\eqref{eq:VTOL_Restdot}, one finds
	\begin{align}
		& -\frac{1}{2\delta_{o1}}\frac{d}{dt}\boldsymbol{\Upsilon}(\tilde{R}_{o})^{\top}\hat{R}^{\top}\tilde{b}_{\Omega}\leq-\frac{1}{2\delta_{o1}}||\hat{R}^{\top}\tilde{b}_{\Omega}||^{2}\nonumber \\
		& \hspace{7em}+\frac{c_{o1}}{\delta_{o1}}||\hat{R}^{\top}\tilde{b}_{\Omega}||\sqrt{||\tilde{R}_{o}||_{{\rm I}}}+\frac{c_{o1}}{\delta_{o1}}||\tilde{R}_{o}||_{{\rm I}}\label{eq:VTOL_R_LyapC1_Aux}
	\end{align}
	where $\dot{\tilde{b}}_{\Omega}=-\dot{\hat{b}}_{\Omega}=-\gamma_{o}\hat{R}\boldsymbol{\Upsilon}(\tilde{R}_{o}M)$,
	$\eta_{\Omega}=\sup_{t\geq0}||\Omega||$, and $\text{\ensuremath{c_{o1}}=\ensuremath{\max}\{2(\ensuremath{\eta_{b}}+\ensuremath{\sqrt{3}k_{o1}}+\ensuremath{\overline{\lambda}_{M}})+\ensuremath{\eta_{\Omega}},2\ensuremath{\gamma_{o}\overline{\lambda}_{M}}\}}$.
	Define the following Lyapunov function candidate $\mathcal{L}_{o1}:\mathbb{SO}\left(3\right)\times\mathbb{R}^{3}\rightarrow\mathbb{R}_{+}$:
	\begin{equation}
		\mathcal{L}_{o1}=\frac{1}{2}{\rm Tr}\{(\mathbf{I}_{3}-\tilde{R}_{o})M\}+\frac{1}{2\gamma_{o}}\tilde{b}_{\Omega}^{\top}\tilde{b}_{\Omega}-\frac{1}{2\delta_{o1}}\boldsymbol{\Upsilon}(\tilde{R}_{o})^{\top}\hat{R}^{\top}\tilde{b}_{\Omega}\label{eq:VTOL_Lyap_Lo1}
	\end{equation}
	In view of \eqref{eq:VTOL_lemm_vexRIM}, one has
	\[
	e_{o1}^{\top}\underbrace{\left[\begin{array}{cc}
			\underline{\lambda}_{\overline{M}} & -\frac{1}{2\delta_{o1}}\\
			-\frac{1}{2\delta_{o1}} & \frac{1}{2\gamma_{o}}
		\end{array}\right]}_{Q_{1}}e_{o1}\leq\mathcal{L}_{o1}\leq e_{o1}^{\top}\underbrace{\left[\begin{array}{cc}
			\overline{\lambda}_{\overline{M}} & \frac{1}{2\delta_{o1}}\\
			\frac{1}{2\delta_{o1}} & \frac{1}{2\gamma_{o}}
		\end{array}\right]}_{Q_{2}}e_{o1}
	\]
	where $e_{o1}=[\sqrt{||\tilde{R}_{o}||_{{\rm I}}},||\hat{R}^{\top}\tilde{b}_{\Omega}||]^{\top}$.
	It becomes evident that $Q_{1}$ and $Q_{2}$ can be made positive
	by selecting $\delta_{o1}>\sqrt{\frac{\gamma_{o}}{2\underline{\lambda}_{\overline{M}}}}$.
	From \eqref{eq:VTOL_Omestdot} and \eqref{eq:VTOL_RerrMI_dot}, one
	obtains
	\begin{align}
		\dot{C}_{o} & =\frac{1}{2}\frac{d}{dt}{\rm Tr}\{(\mathbf{I}_{3}-\tilde{R}_{o})M\}+\frac{1}{2\gamma_{o}}\frac{d}{dt}\tilde{b}_{\Omega}^{\top}\tilde{b}_{\Omega}\nonumber \\
		& =-\boldsymbol{\Upsilon}(\tilde{R}_{o}M)^{\top}(w_{\Omega}-\hat{R}^{\top}\tilde{b}_{\Omega})-\frac{1}{\gamma_{o}}\tilde{b}_{\Omega}^{\top}\dot{\hat{b}}_{\Omega}\nonumber \\
		& =-k_{o1}||\boldsymbol{\Upsilon}(\tilde{R}_{o}M)||^{2}\nonumber \\
		& \leq-k_{o1}\underline{\lambda}_{\overline{M}}^{2}||\tilde{R}_{o}||_{{\rm I}}\label{eq:VTOL_Lyap_Lo1dot-1}
	\end{align}
	by employing \eqref{eq:VTOL_lemm_vexRIM} in Lemma \ref{Lemm:vex_RI}.
	The result in \eqref{eq:VTOL_Lyap_Lo1dot-1} reveals that $\dot{C}_{o}$
	is negative and continuous, $\dot{C}_{o}\rightarrow0$ such that $C_{o}\in\mathcal{L}_{\infty}$,
	and a finite $\lim_{t\rightarrow\infty}C_{o}$ exists. Hence, $\tilde{R}_{o}$
	and $\tilde{b}_{\Omega}$ are globally bounded. Therefore, $\ddot{C}_{o}$
	is bounded, and $\lim_{t\rightarrow\infty}k_{o1}||\boldsymbol{\Upsilon}(\tilde{R}_{o}M)||=0_{3\times1}$
	indicates that $\lim_{t\rightarrow\infty}\tilde{R}_{o}=\mathbf{I}_{3}$.
	The boundedness of $\tilde{b}_{\Omega}$ implies that $\ddot{\tilde{R}}_{o}$
	is bounded, and on the basis of Barbalat Lemma \ref{lem:Barbalat},
	$\lim_{t\rightarrow\infty}\dot{\tilde{R}}_{o}=0_{3\times3}$. Since
	$\lim_{t\rightarrow\infty}||\boldsymbol{\Upsilon}(\tilde{R}_{o}M)||=0_{3\times1}$,
	it becomes apparent that $\lim_{t\rightarrow\infty}w_{\Omega}=0_{3\times1}$.
	Since $\lim_{t\rightarrow\infty}\dot{\tilde{R}}_{o}=0_{3\times3}$,
	$\lim_{t\rightarrow\infty}-[w_{\Omega}-\hat{R}^{\top}\tilde{b}_{\Omega}]_{\times}\tilde{R}_{o}=0_{3\times3}$
	showing that $\lim_{t\rightarrow\infty}\tilde{b}_{\Omega}=0_{3\times1}$
	and, in turn, $\lim_{t\rightarrow\infty}C_{o}=0$ starting from almost
	any initial condition. Using \eqref{eq:VTOL_R_LyapC1_Aux}, the derivative
	of \eqref{eq:VTOL_Lyap_Lo1} is:
	\begin{align*}
		\dot{\mathcal{L}}_{o1}\leq & -(k_{o1}\underline{\lambda}_{\overline{M}}^{2}-\frac{c_{o1}}{\delta_{o1}})||\tilde{R}_{o}||_{{\rm I}}-\frac{1}{2\delta_{o1}}||\hat{R}^{\top}\tilde{b}_{\Omega}||^{2}\\
		& +\frac{c_{o1}}{\delta_{o1}}||\hat{R}^{\top}\tilde{b}_{\Omega}||\sqrt{||\tilde{R}_{o}||_{{\rm I}}}
	\end{align*}
	such that
	\begin{align}
		\dot{\mathcal{L}}_{o1} & \leq-\frac{1}{2\delta_{o1}}e_{o1}^{\top}\underbrace{\left[\begin{array}{cc}
				2k_{o1}\delta_{o1}\underline{\lambda}_{\overline{M}}^{2}-2c_{o1} & c_{o1}\\
				c_{o1} & 1
			\end{array}\right]}_{D_{o1}}e_{o1}\label{eq:VTOL_Lyap_Lo1dot}
	\end{align}
	$D_{o1}$ is made positive by setting $\delta_{o1}>\frac{c_{o1}^{2}+2c_{o1}}{2k_{o1}\underline{\lambda}_{\overline{M}}^{2}}$.
	Thereby, by selecting $\delta_{o1}>\max\{\sqrt{\frac{\gamma_{o}}{2\underline{\lambda}_{\overline{M}}}},\frac{c_{o1}^{2}+2c_{o1}}{2k_{o1}\underline{\lambda}_{\overline{M}}^{2}}\}$
	and letting $\underline{\lambda}_{D_{o1}}$ be the minimum eigenvalue
	of $D_{o1}$, one finds
	\begin{align}
		\dot{\mathcal{L}}_{o1} & \leq-\underline{\lambda}_{D_{o1}}(||\tilde{R}_{o}||_{{\rm I}}+||\hat{R}^{\top}\tilde{b}_{\Omega}||^{2})\leq-\frac{\underline{\lambda}_{D_{o1}}}{\eta_{o}}\mathcal{L}_{o1}\label{eq:VTOL_Lyap_Lo1Final}
	\end{align}
	where $\eta_{o}=\max\{\overline{\lambda}(Q_{1}),\overline{\lambda}(Q_{2})\}$.
	From \eqref{eq:VTOL_yerr_dot} and \eqref{eq:VTOL_Verr_dot}, define
	the Lyapunov function candidate $\mathcal{L}_{o2}:\mathbb{R}^{3}\times\mathbb{R}^{3}\rightarrow\mathbb{R}_{+}$
	as
	\begin{equation}
		\mathcal{L}_{o2}=||\sum_{i=1}^{n}s_{i}\tilde{y}_{i}||^{2}+\frac{1}{2k_{o3}}\tilde{V}_{o}^{\top}\tilde{V}_{o}-\delta_{o2}\tilde{V}_{o}^{\top}\sum_{i=1}^{n}s_{i}\tilde{y}_{i}\label{eq:VTOL_Lyap_Lo2}
	\end{equation}
	One can show that
	\[
	e_{o2}^{\top}\underbrace{\left[\begin{array}{cc}
			\frac{1}{2} & -\frac{\delta_{o2}}{2}\\
			-\frac{\delta_{o2}}{2} & \frac{1}{2k_{o3}}
		\end{array}\right]}_{Q_{3}}e_{o2}\leq\mathcal{L}_{o2}\leq e_{o2}^{\top}\underbrace{\left[\begin{array}{cc}
			\frac{1}{2} & \frac{\delta_{o2}}{2}\\
			\frac{\delta_{o2}}{2} & \frac{1}{2k_{o3}}
		\end{array}\right]}_{Q_{4}}e_{o2}
	\]
	where $e_{o2}=[||\sum_{i=1}^{n}s_{i}\tilde{y}_{i}||,||\tilde{V}_{o}||]^{\top}$.
	One finds that $Q_{3}$ and $Q_{4}$ are made positive by setting
	$\delta_{o2}<\frac{1}{\sqrt{k_{o3}}}$. From \eqref{eq:VTOL_yerr_dot}
	and \eqref{eq:VTOL_Verr_dot}, one finds
	\begin{align}
		& \dot{\mathcal{L}}_{o2}\leq-e_{o2}^{\top}\underbrace{\left[\begin{array}{cc}
				s_{T}k_{o2}-\delta_{o2}k_{o3} & \frac{\delta_{o2}(s_{T}k_{o2}+2k_{o1}\overline{\lambda}_{M})}{2}\\
				\frac{\delta_{o2}(s_{T}k_{o2}+2k_{o1}\overline{\lambda}_{M})}{2} & s_{T}\delta_{o2}
			\end{array}\right]}_{D_{o2}}e_{o2}\nonumber \\
		& \hspace{1em}+c_{o2}(||\tilde{V}_{o}||+||\sum_{i=1}^{n}s_{i}\tilde{y}_{i}||)(||\hat{R}^{\top}\tilde{b}_{\Omega}||+\sqrt{||\tilde{R}_{o}||_{{\rm I}}})\label{eq:VTOL_LyapL2_dot}
	\end{align}
	where $\eta_{P}=\sup_{t\geq0}||P-p_{c}||$, $\eta_{V}=\sup_{t\geq0}||V||$,
	and $c_{o2}=\max\{s_{T}\eta_{P}\delta_{o2}+\frac{1}{k_{o3}},\eta_{V}\delta_{o2}+s_{T}\eta_{P},\frac{g}{k_{o3}},\delta_{o2}g\}$.
	It should be noted that $||\mathbf{I}_{3}-\tilde{R}_{o}||_{F}=2\sqrt{2}\sqrt{||\tilde{R}_{o}||_{{\rm I}}}$
	\cite{hashim2019AtiitudeSurvey}. It is evident that $D_{o2}$ is
	positive if $\delta_{o2}<\frac{4s_{T}^{2}k_{o2}}{4s_{T}k_{o3}+(s_{T}k_{o2}+2k_{o1}\overline{\lambda}_{M})^{2}}$.
	As such, let us select $\delta_{o2}<\min\{\frac{1}{\sqrt{k_{o3}}},\frac{4s_{T}^{2}k_{o2}}{4s_{T}k_{o3}+(s_{T}k_{o2}+2k_{o1}\overline{\lambda}_{M})^{2}}\}$,
	and let $\underline{\lambda}_{D_{o2}}$ be the minimum eigenvalue
	of $D_{o2}$. One can show that
	\begin{equation}
		\dot{\mathcal{L}}_{o2}\leq-\underline{\lambda}_{D_{o2}}||e_{o2}||^{2}+c_{o2}||e_{o1}||\,||e_{o2}||\label{eq:VTOL_Lyap_Lo2Final}
	\end{equation}
	From \eqref{eq:VTOL_Lyap_Lo1} and \eqref{eq:VTOL_Lyap_Lo2}, define
	the following Lyapunov function candidate $\mathcal{L}_{oT}:\mathbb{SO}\left(3\right)\times\mathbb{R}^{3}\times\mathbb{R}^{3}\times\mathbb{R}^{3}\rightarrow\mathbb{R}_{+}$:
	\begin{equation}
		\mathcal{L}_{oT}=\mathcal{L}_{o1}+\mathcal{L}_{o2}\label{eq:VTOL_LyapLo_Total}
	\end{equation}
	From \eqref{eq:VTOL_Lyap_Lo1Final} and \eqref{eq:VTOL_LyapL2_dot},
	one obtains
	\[
	\dot{\mathcal{L}}_{oT}\leq-\underline{\lambda}_{D_{o1}}||e_{o1}||^{2}-\underline{\lambda}_{D_{o2}}||e_{o2}||^{2}+c_{o2}||e_{o1}||\,||e_{o2}||
	\]
	\begin{equation}
		\dot{\mathcal{L}}_{oT}\leq-{\bf e}_{o}^{\top}\underbrace{\left[\begin{array}{cc}
				\underline{\lambda}_{D_{o1}} & \frac{1}{2}c_{o2}\\
				\frac{1}{2}c_{o2} & \underline{\lambda}_{D_{o2}}
			\end{array}\right]}_{D_{o}}{\bf e}_{o}\label{eq:VTOL_LyapLodot_Total}
	\end{equation}
	with ${\bf e}_{o}=[||e_{o1}||,||e_{o2}||]^{\top}$ and $D_{o}$ being
	positive if $\underline{\lambda}_{D_{o1}}>\frac{1}{4\underline{\lambda}_{D_{o2}}}c_{o2}^{2}$.
	Let $\underline{\lambda}_{D_{o}}$ be the minimum eigenvalue of $D_{o}$
	and let $\eta_{o}=\max\{\overline{\lambda}(Q_{1}),\overline{\lambda}(Q_{2}),\overline{\lambda}(Q_{3}),\overline{\lambda}(Q_{4})\}$.
	It becomes apparent that
	\begin{equation}
		\mathcal{L}_{oT}(t)\leq\mathcal{L}_{oT}(0)\exp(-\underline{\lambda}_{D_{o}}t/\eta_{o})\label{eq:VTOL_Lo_Total-2}
	\end{equation}
	As such, $\lim_{t\rightarrow\infty}\tilde{R}_{o}=\mathbf{I}_{3}$
	and $\lim_{t\rightarrow\infty}||\tilde{b}_{\Omega}||=\lim_{t\rightarrow\infty}||\sum_{i=1}^{n}s_{i}\tilde{y}_{i}||=\lim_{t\rightarrow\infty}||\tilde{V}_{o}||=0$
	exponentially. Thus, $||\sum_{i=1}^{n}s_{i}\tilde{y}_{i}||\rightarrow0$
	implies that $||\tilde{P}_{o}||\rightarrow0$ exponentially, see \eqref{eq:VTOL_ye}.
	Therefore, the closed-loop error signals are almost globally exponentially
	stable proving Theorem \ref{thm:Theorem1}.\end{proof}

\section{Direct Observer-based Controller Design \label{sec:VTOL_Controller}}

Let $R_{d}\in\mathbb{SO}\left(3\right)$, $\Omega_{d}\in\mathbb{R}^{3}$,
$P_{d}\in\mathbb{R}^{3}$, and $V_{d}\in\mathbb{R}^{3}$ denote the
desired vehicle's orientation, angular velocity, position, and linear
velocity, respectively. This Section aims to design an observer-based
controller for a VTOL-UAV that tracks the true VTOL-UAV motion components
($R$, $\Omega$, $P$, and $V$) along the desired trajectory ($R_{d}$,
$\Omega_{d}$, $P_{d}$, and $V_{d}$) given the information estimated
($\hat{R}$, $\hat{\Omega}$, $\hat{P}$, and $\hat{V}$) by the novel
direct observer as described in the previous Section along with the
set of feature observation and measurement in \eqref{eq:VTOL_VRP}
and the control inputs $\mathcal{T}\in\mathbb{R}^{3}$ and $\Im\in\mathbb{R}$
defined in \eqref{eq:VTOL_Rotation} and \eqref{eq:VTOL_Translation}.
Let us define the following error components:
\begin{align}
	\tilde{R}_{c}= & R_{d}^{\top}R\label{eq:VTOL_Rerr-c}\\
	\tilde{\Omega}_{c}= & R_{d}^{\top}(\Omega_{d}-\Omega)\label{eq:VTOL_Omerr-c}\\
	\tilde{P}_{c}= & P-P_{d}\label{eq:VTOL_Perr-c}\\
	\tilde{V}_{c}= & V-V_{d}\label{eq:VTOL_Verr-c}
\end{align}
The objective of the proposed control laws is to drive $R\rightarrow R_{d}$,
$\Omega\rightarrow\Omega_{d}$, $P\rightarrow P_{d}$, and $V\rightarrow V_{d}$
such that $\lim_{t\rightarrow\infty}\tilde{R}_{c}=\mathbf{I}_{3}$
and $\lim_{t\rightarrow\infty}\tilde{\Omega}_{c}=\lim_{t\rightarrow\infty}\tilde{P}_{c}=\lim_{t\rightarrow\infty}\tilde{V}_{c}=0_{3\times1}$.
Recalling that $s_{T}=\sum_{i=1}^{n}s_{i}$, one finds
\begin{align}
	\tilde{R}_{c}M & =R_{d}^{\top}RM=R_{d}^{\top}\hat{R}\hat{R}^{\top}RM=R_{d}^{\top}\hat{R}\tilde{R}_{o}M\nonumber \\
	& =R_{d}^{\top}\sum_{i=1}^{n}s_{i}y_{i}(p_{i}-p_{c})^{\top}\label{eq:VTOL_RdM}
\end{align}
where $\tilde{R}_{o}M=\sum_{i=1}^{n}s_{i}\hat{R}^{\top}y_{i}(p_{i}-p_{c})^{\top}$
as defined in \eqref{eq:VTOL_Measurements}. Thereby, one shows
\begin{align}
	& \boldsymbol{\mathcal{P}}_{a}(\tilde{R}_{c}M)=\boldsymbol{\mathcal{P}}_{a}(\sum_{i=1}^{n}s_{i}R_{d}^{\top}y_{i}(p_{i}-p_{c})^{\top})\nonumber \\
	& \hspace{1.5em}=\sum_{i=1}^{n}\frac{s_{i}}{2}\left(R_{d}^{\top}y_{i}(p_{i}-p_{c})^{\top}-(p_{i}-p_{c})y_{i}^{\top}R_{d}\right)\label{eq:VTOL_Pa_RM-1}
\end{align}
In view of \eqref{eq:VTOL_VEX}, $\boldsymbol{\Upsilon}(\tilde{R}_{c}M)=\mathbf{vex}(\boldsymbol{\mathcal{P}}_{a}(\tilde{R}_{c}M))$.
Therefore, one has
\begin{align}
	\boldsymbol{\Upsilon}(\tilde{R}_{c}M) & =\sum_{i=1}^{n}\frac{s_{i}}{2}\left((p_{i}-p_{c})\times R_{d}^{\top}y_{i}\right)\label{eq:VTOL_VEX_RM-1}
\end{align}
In view of the true rotational dynamics in \eqref{eq:VTOL_Rotation},
the desired attitude dynamics are
\begin{equation}
	\dot{R}_{d}=-[\Omega_{d}]_{\times}R_{d}\label{eq:VTOL_Rd_dot}
\end{equation}
Recall that $\dot{V}=ge_{3}-\frac{\Im}{m}R^{\top}e_{3}$ in \eqref{eq:VTOL_Translation}
and re-express it $\dot{V}=ge_{3}-\frac{\Im}{m}R_{d}^{\top}e_{3}-\frac{\Im}{m}(R^{\top}-R_{d}^{\top})e_{3}=F-\frac{\Im}{m}(R^{\top}-R_{d}^{\top})e_{3}$
with $F$ standing for an intermediary control input defined by
\begin{equation}
	F=ge_{3}-\frac{\Im}{m}R_{d}^{\top}e_{3}=[f_{1},f_{2},f_{3}]^{\top}\in\mathbb{R}^{3}\label{eq:VTOL_F_1}
\end{equation}
with $F=[f_{1},f_{2},f_{3}]^{\top}\in\mathbb{R}^{3}$ denoting an
intermediary control input to the translational dynamics described
in \eqref{eq:VTOL_Translation} which shows that $\Im=m||ge_{3}-F||$.
Thus, the desired attitude and angular velocity will be obtained with
the aid of $F$ which allows for an explicit extraction of the desired
angular velocity $\Omega_{d}$ and its rate of change $\dot{\Omega}_{d}$.

\begin{assum}\label{Assum:P_desired}The desired position is upper-bounded,
	and its first $\dot{P}_{d}=V_{d}$, second $\ddot{P}_{d}$, third
	$P_{d}^{(3)}$, and fourth $P_{d}^{(4)}$ time-derivatives are upper-bounded.
	Also, $\Omega_{d}$ and $\dot{\Omega}_{d}$ are upper-bounded by a
	scalar $\gamma_{d}<\infty$ such that $\gamma_{d}\geq\max\{\sup_{t\geq0}||\Omega_{d}||,\sup_{t\geq0}||\dot{\Omega}_{d}||\}$.\end{assum}
\begin{lem}
	\label{lem:Lemma_Qd}\cite{Roberts2009} Consider the linear velocity
	dynamics in \eqref{eq:VTOL_Translation}, and let $F=[f_{1},f_{2},f_{3}]^{\top}\in\mathbb{R}^{3}$.
	The thrust magnitude is $\Im=m||ge_{3}-F||$, and the desired components
	of the unit-quaternion $Q_{d}=[q_{d0},q_{d}^{\top}]^{\top}\in\mathbb{S}^{3}$
	are as follows:
	\begin{equation}
		q_{d0}=\sqrt{\frac{m}{2\Im}(g-f_{3})+\frac{1}{2}},\hspace{1em}q_{d}=\left[\begin{array}{c}
			\frac{m}{2\Im q_{d0}}f_{2}\\
			-\frac{m}{2\Im q_{d0}}f_{1}\\
			0
		\end{array}\right]\label{eq:VTOL_Q_Qd}
	\end{equation}
	given that $F\neq[0,0,c]^{\top}$ for $c\geq g$ and $\mathbb{S}^{3}=\{Q_{d}\in\mathbb{R}^{4}|\,||Q_{d}||=1\}$.
	Assume $F$ is differentiable. Consequently, the desired angular velocity
	$\Omega_{d}$ is
	\begin{align}
		\Omega_{d} & =\Xi(F)\dot{F}\label{eq:VTOL_Q_Omd}
	\end{align}
	with
	\begin{equation}
		\Xi(F)=\frac{1}{\alpha_{1}^{2}\alpha_{2}}\left[\begin{array}{ccc}
			-f_{1}f_{2} & -f_{2}^{2}+\alpha_{1}\alpha_{2} & f_{2}\alpha_{2}\\
			f_{1}^{2}-\alpha_{1}\alpha_{2} & f_{1}f_{2} & -f_{1}\alpha_{2}\\
			f_{2}\alpha_{1} & -f_{1}\alpha_{1} & 0
		\end{array}\right]\label{eq:VTOL_Q_Lambda}
	\end{equation}
	such that $\alpha_{1}=||ge_{3}-F||$ and $\alpha_{2}=||ge_{3}-F||+g-f_{3}$.
\end{lem}
Lemma \ref{lem:Lemma_Qd} implies that $\Im$ and $Q_{d}$ in \eqref{eq:VTOL_Q_Qd}
are singularity-free. In view of Lemma \ref{lem:Lemma_Qd}, the desired
orientation $R_{d}$ is as follows \cite{hashim2019AtiitudeSurvey,shuster1993survey}:
\[
R_{d}=(q_{d0}^{2}-||q_{d}||^{2})\mathbf{I}_{3}+2q_{d}q_{d}^{\top}-2q_{d0}[q_{d}]_{\times}\in\mathbb{SO}\left(3\right)
\]

\begin{rem}
	$F$ is designed to be twice differentiable in order for the desired
	rate of change of angular velocity $\dot{\Omega}_{d}$ to be as follows:
	\begin{equation}
		\dot{\Omega}_{d}=\dot{\Xi}(F)\dot{F}+\Xi(F)\ddot{F}\label{eq:VTOL_Q_Omd_dot-1}
	\end{equation}
	The derivation of $\dot{F}$ and $\ddot{F}$ are subsequently provided.
\end{rem}
Consider introducing the following variables:
\begin{equation}
	\mathcal{E}=\tilde{P}_{c}-\theta,\hspace{2em}\dot{\mathcal{E}}=\tilde{V}_{c}-\dot{\theta}\label{eq:VTOL_Q_E}
\end{equation}
with $\theta\in\mathbb{R}^{3}$ describing an adaptively tuned auxiliary
variable. The adaptation mechanism of $\theta$ and the intermediary
control input ($F$) are designed as follows:
\begin{equation}
	\begin{cases}
		\ddot{\theta} & =-k_{\theta1}\tanh(\theta)-k_{\theta2}\tanh(\dot{\theta})\\
		& \hspace{1em}+k_{c3}(\hat{P}-P_{d}-\theta)+k_{c4}(\hat{V}-V_{d}-\dot{\theta})\\
		F & =\ddot{P}_{d}-k_{\theta1}\tanh(\theta)-k_{\theta2}\tanh(\dot{\theta})
	\end{cases}\label{eq:VTOL_F_theta}
\end{equation}
where $k_{\theta1}$, $k_{\theta2}$, $k_{c3}$, and $k_{c4}$ denote
positive constants. $F$ is selected as in \cite{abdessameud2010global}.
In view of \eqref{eq:VTOL_F_theta}, the first and the second derivatives
of $F$ are
\begin{equation}
	\begin{cases}
		\dot{F} & =P_{d}^{(3)}-k_{\theta1}H\dot{\theta}-k_{\theta2}\dot{H}\ddot{\theta}\\
		\ddot{F} & =P_{d}^{(4)}-k_{\theta1}Z-k_{\theta2}\dot{Z}
	\end{cases}\label{eq:VTOL_Fdot}
\end{equation}
where $\hbar(\theta_{i})=1-\tanh^{2}(\theta_{i})$, $\hbar(\dot{\theta}_{i})=1-\tanh^{2}(\dot{\theta}_{i})$,
$z_{i}=z(\theta_{i},\dot{\theta}_{i},\ddot{\theta}_{i})=\hbar(\theta_{i})(\ddot{\theta}_{i}-2\tanh(\theta_{i})\dot{\theta}_{i}^{2})$,
and $\dot{z}_{i}=z(\dot{\theta}_{i},\ddot{\theta}_{i},\theta_{i}^{(3)})=\hbar(\dot{\theta}_{i})(\theta_{i}^{(3)}-2\tanh(\dot{\theta}_{i})\ddot{\theta}_{i}^{2})$
for all $i=1,2,3$ such that $H={\rm diag}(\hbar(\theta_{1}),\hbar(\theta_{2}),\hbar(\theta_{3}))$,
$\dot{H}={\rm diag}(\hbar(\dot{\theta}_{1}),\hbar(\dot{\theta}_{2}),\hbar(\dot{\theta}_{3}))$,
$Z=[z_{1},z_{2},z_{3}]^{\top}$, and $\dot{Z}=[\dot{z}_{1},\dot{z}_{2},\dot{z}_{3}]^{\top}$.
Also, the third derivative $\beta^{(3)}$ is equivalent to
\[
\theta^{(3)}=-k_{\beta1}H\dot{\theta}-k_{\theta2}\dot{H}\ddot{\theta}+k_{c1}(\dot{\hat{P}}-\dot{P}_{d}-\theta)+k_{c2}(\dot{\hat{V}}-\dot{V}_{d}-\ddot{\theta})
\]
In view of Lemma \ref{lem:Lemma_Qd}, one has $\dot{\alpha}_{1}=\frac{1}{\alpha_{1}}[f_{1},f_{2},(f_{3}-g)]^{\top}\dot{F}$
and $\dot{\alpha}_{2}=\dot{\alpha}_{1}-\dot{f}_{3}$ with $\dot{F}=[\dot{f}_{1},\dot{f}_{2},\dot{f}_{3}]^{\top}$.
As such, it is straight forward to obtain $\dot{\Omega}_{d}=\dot{\Xi}(F)\dot{F}+\Xi(F)\ddot{F}$.

To this end, let us propose the following control laws based on the
direct measurements and estimates of attitude, position, gyro bias,
and linear velocity:
\begin{align}
	& \begin{cases}
		\boldsymbol{\Upsilon}(\tilde{R}_{c}M) & =\sum_{i=1}^{n}\frac{s_{i}}{2}\left((p_{i}-p_{c})\times R_{d}^{\top}y_{i}\right)\\
		w_{c} & =k_{c1}R_{d}\boldsymbol{\Upsilon}(\tilde{R}_{c}M)+k_{c2}(\Omega_{d}-\Omega_{m}+\hat{b}_{\Omega})\\
		\mathcal{T} & =w_{c}+J\dot{\Omega}_{d}-\left[J(\Omega_{m}-\hat{b}_{\Omega})\right]_{\times}\Omega_{d}
	\end{cases}\label{eq:VTOL_Tau}\\
	& \begin{cases}
		\ddot{\theta} & =-k_{\theta1}\tanh(\theta)-k_{\theta2}\tanh(\dot{\theta})\\
		& \hspace{1em}+k_{c3}(\hat{P}-P_{d}-\theta)+k_{c4}(\hat{V}-V_{d}-\dot{\theta})\\
		F & =\ddot{P}_{d}-k_{\theta1}\tanh(\theta)-k_{\theta2}\tanh(\dot{\theta})\\
		\Im & =m||ge_{3}-F||
	\end{cases}\label{eq:VTOL_Thrust}
\end{align}
where $\mathcal{T}$, $J$, $\Im$, $F$, $g$, and $m$ are torque
input, inertia matrix, thrust magnitude, intermediary control input,
gravity, and mass, respectively. $\Omega_{m}$ and $\theta\in\mathbb{R}^{3}$
stand for a gyro measurement and an auxiliary variable, respectively,
while $\hat{b}_{\Omega}$, $\hat{P}$, and $\hat{V}$ define estimates
of gyro bias, position, and linear velocity, respectively. $k_{c1}$,
$k_{c2}$, $k_{c3}$, $k_{c4}$, $k_{\theta1}$, and $k_{\theta2}$
represent positive constants. Fig. \ref{fig:VTOL-1} presents a conceptual
summary of the proposed methodology.
\begin{figure*}
	\centering{}\includegraphics[scale=0.35]{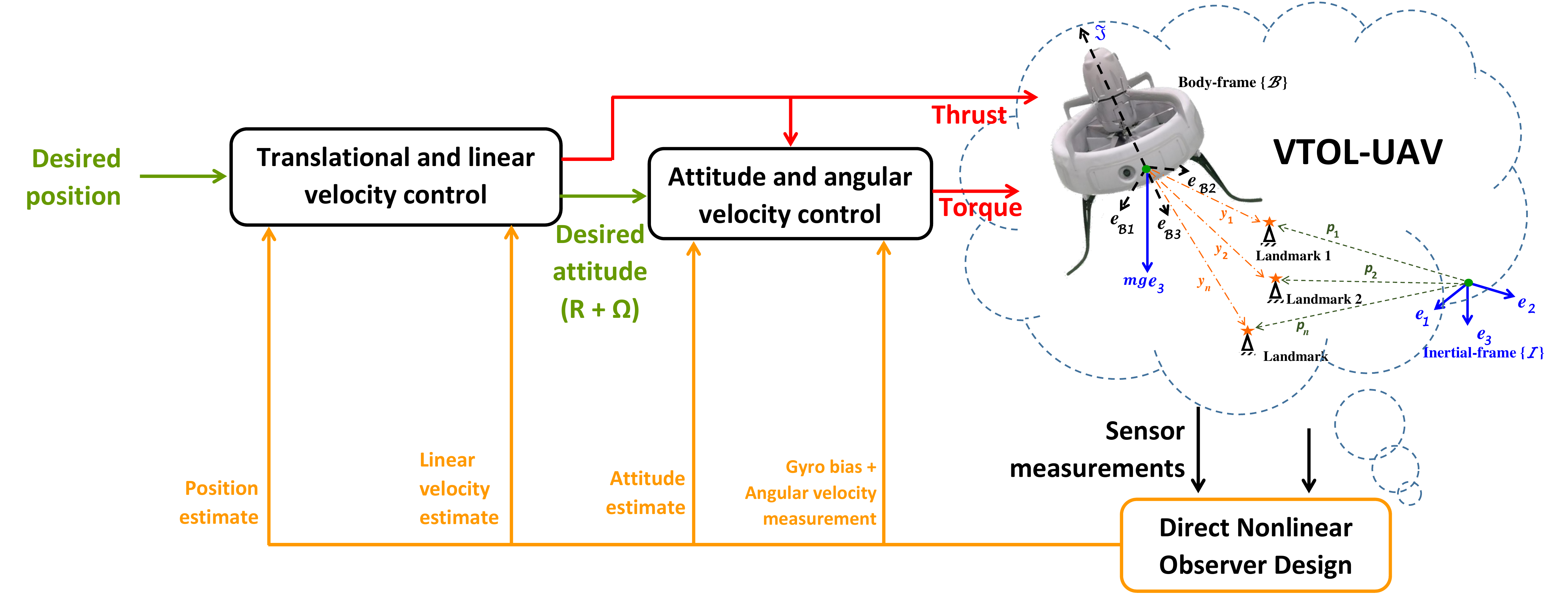}\caption{Illustrative diagram of the proposed observer-based controller for
		VTOL-UAV.}
	\label{fig:VTOL-1}
\end{figure*}

\begin{thm}
	\label{thm:Theorem2}Recall the dynamics in \eqref{eq:VTOL_Rotation}
	and \eqref{eq:VTOL_Translation} and the direct observer design defined
	in \eqref{eq:VTOL_ObsvCompact}. Let Assumption \ref{Assum:P_desired}
	be met and consider the control laws in \eqref{eq:VTOL_Tau} and \eqref{eq:VTOL_Thrust}.
	Then for $\tilde{R}_{o}(0)\notin\mathcal{S}_{u}$ and $\tilde{R}_{c}(0)\notin\mathcal{S}_{u}$,
	$\lim_{t\rightarrow\infty}\tilde{R}_{c}=\tilde{R}_{o}=\mathbf{I}_{3}$,
	$\lim_{t\rightarrow\infty}\tilde{b}_{\Omega}=\sum_{i=1}^{n}s_{i}\tilde{y}_{i}=\tilde{V}_{o}=0_{3\times1}$,
	$\lim_{t\rightarrow\infty}\tilde{\Omega}_{c}=\lim_{t\rightarrow\infty}\tilde{P}_{c}=\lim_{t\rightarrow\infty}\tilde{V}_{c}=0_{3\times1}$,
	and the observer-based controller closed-loop error signals are uniformly
	almost globally exponentially stable.
\end{thm}
\begin{proof}From \eqref{eq:VTOL_Rerr-c}, \eqref{eq:VTOL_Rotation}
	and \eqref{eq:VTOL_Rd_dot}, one obtains
	\begin{align}
		\dot{\tilde{R}}_{c} & =\dot{R}_{d}^{\top}R+R_{d}^{\top}\dot{R}=[R_{d}^{\top}(\Omega_{d}-\Omega)]_{\times}\tilde{R}_{c}\nonumber \\
		& =[\tilde{\Omega}_{c}]_{\times}\tilde{R}_{c}\label{eq:VTOL_Rcerr_dot}
	\end{align}
	where $R_{d}^{\top}[\Omega]_{\times}R_{d}=[R_{d}^{\top}\Omega]_{\times}$.
	Let us define $||\tilde{R}_{c}||_{{\rm I}}=\frac{1}{4}{\rm Tr}\{\mathbf{I}_{3}-\tilde{R}_{c}\}$
	and $||\tilde{R}_{c}M||_{{\rm I}}=\frac{1}{4}{\rm Tr}\{(\mathbf{I}_{3}-\tilde{R}_{c})M\}$.
	From \eqref{eq:VTOL_Ecul_Dist}, \eqref{eq:VTOL_R_Identity2}, and
	\eqref{eq:VTOL_Rcerr_dot}, it becomes apparent that \cite{hashim2019AtiitudeSurvey}
	\begin{align}
		\frac{d}{dt}||\tilde{R}_{c}||_{{\rm I}}= & \frac{1}{2}\boldsymbol{\Upsilon}(\tilde{R}_{c})^{\top}\tilde{\Omega}_{c}\label{eq:VTOL_RerrIc_dot}\\
		\frac{d}{dt}||\tilde{R}_{c}M||_{{\rm I}}= & \frac{1}{2}\boldsymbol{\Upsilon}(\tilde{R}_{c}M)^{\top}\tilde{\Omega}_{c}\label{eq:VTOL_RerrIcM_dot}
	\end{align}
	Using \eqref{eq:VTOL_Rotation}, \eqref{eq:VTOL_Omerr-c}, \eqref{eq:VTOL_Rcerr_dot},
	and $\mathcal{T}$ in \eqref{eq:VTOL_Tau}, one finds
	\begin{align}
		\frac{d}{dt}JR_{d}\tilde{\Omega}_{c}= & J\dot{\Omega}_{d}-J\dot{\Omega}\nonumber \\
		= & J\dot{\Omega}_{d}-[J\Omega]_{\times}R_{d}\tilde{\Omega}_{c}-[J\Omega]_{\times}\Omega_{d}-\mathcal{T}\nonumber \\
		= & -[J\Omega]_{\times}R_{d}\tilde{\Omega}_{c}-[\Omega_{d}]_{\times}J\tilde{b}_{\Omega}-w_{c}\label{eq:VTOL_OmerrC_dot}
	\end{align}
	Using \eqref{eq:VTOL_Perr-c}, \eqref{eq:VTOL_Verr-c}, and \eqref{eq:VTOL_Translation},
	position and velocity errors are as follows:
	\begin{equation}
		\begin{cases}
			\dot{\tilde{P}}_{c} & =\tilde{V}_{c}\\
			\dot{\tilde{V}}_{c} & =ge_{3}-\frac{\Im}{m}R^{\top}e_{3}-\dot{V}_{d}
		\end{cases}\label{eq:VTOL_P_V_errC_dot}
	\end{equation}
	As such, the derivatives of the variables in \eqref{eq:VTOL_Q_E}
	are equivalent to
	\begin{equation}
		\begin{cases}
			\dot{\mathcal{E}} & =\tilde{V}_{c}-\dot{\theta}\\
			\ddot{\mathcal{E}} & =F-||ge_{3}-F||(R^{\top}-R_{d}^{\top})e_{3}-\ddot{P}_{d}-\ddot{\theta}
		\end{cases}\label{eq:VTOL_E_dot}
	\end{equation}
	Let us define $C_{c}=\frac{1}{2}{\rm Tr}\{(\mathbf{I}_{3}-\tilde{R}_{c})M\}+\frac{1}{2k_{c1}}\tilde{\Omega}_{c}^{\top}R_{d}^{\top}JR_{d}\tilde{\Omega}_{c}$.
	From \eqref{eq:VTOL_Omestdot} and \eqref{eq:VTOL_RerrMI_dot}, one
	obtains
	\begin{align}
		\dot{C}_{c}= & \boldsymbol{\Upsilon}(\tilde{R}_{c}M)^{\top}\tilde{\Omega}_{c}\nonumber \\
		& -\frac{1}{k_{c1}}\tilde{\Omega}_{c}^{\top}R_{d}^{\top}([J\Omega]_{\times}R_{d}\tilde{\Omega}_{c}+[\Omega_{d}]_{\times}J\tilde{b}_{\Omega}+w_{c})\nonumber \\
		\leq & -\frac{k_{c2}}{k_{c1}}||\tilde{\Omega}_{c}||^{2}+\frac{k_{c2}+\overline{\lambda}_{J}\eta_{d}}{k_{c1}}||\tilde{\Omega}_{c}||\,||\tilde{b}_{\Omega}||\label{eq:VTOL_LyapLc2-1}
	\end{align}
	where $\eta_{d}=\sup_{t\geq0}\{||\Omega_{d}||\}$ and $\tilde{\Omega}_{c}^{\top}R_{d}^{\top}[J\Omega]_{\times}R_{d}\tilde{\Omega}_{c}=0$.
	In view of \eqref{eq:VTOL_Lo_Total-2}, $||\tilde{b}_{\Omega}||$
	is bounded and $||\tilde{b}_{\Omega}||\rightarrow0$. As such, $||\tilde{\Omega}_{c}||$
	is bounded. The derivative of $\boldsymbol{\Upsilon}(\tilde{R}_{c})$
	is as follows \cite{hashim2019AtiitudeSurvey}:
	\begin{align}
		\boldsymbol{\Upsilon}(\dot{\tilde{R}}_{c}) & =\frac{1}{2}\Psi(\tilde{R}_{c})\tilde{\Omega}_{c}\label{eq:VTOL_vex_dot-1}
	\end{align}
	with $\Psi(\tilde{R}_{c})={\rm Tr}\{\tilde{R}_{c}\}\mathbf{I}_{3}-\tilde{R}_{c}$.
	Accordingly, one finds
	\begin{align}
		& \frac{1}{\delta_{c1}}\frac{d}{dt}\boldsymbol{\Upsilon}(\tilde{R}_{c})^{\top}\tilde{\Omega}_{c}\nonumber \\
		& \leq-\frac{k_{c1}c_{c2}}{\delta_{c1}}||\tilde{R}_{c}||_{{\rm I}}+(\frac{c_{c3}}{\delta_{c1}}||\tilde{\Omega}_{c}||+\frac{c_{c3}}{\delta_{c1}}||\tilde{b}_{\Omega}||)\sqrt{||\tilde{R}_{c}||_{{\rm I}}}\label{eq:VTOL_vex_eq2}
	\end{align}
	$\eta_{\Omega_{c}}=\sup_{t\geq0}\{||\tilde{\Omega}_{c}||\}$, $c_{c2}=\underline{\lambda}_{\overline{M}}^{2}/\overline{\lambda}_{J}$,
	and $c_{c3}=\max\{\frac{\eta_{\Omega_{c}}}{2}+c_{c1}\overline{\lambda}_{\overline{M}},k_{c2}+\overline{\lambda}_{J}\eta_{d}\}$.
	Define the following Lyapunov function candidate $\mathcal{L}_{c1}:\mathbb{SO}\left(3\right)\times\mathbb{R}^{3}\rightarrow\mathbb{R}_{+}$:
	\begin{equation}
		\mathcal{L}_{c1}=2||\tilde{R}_{c}M||_{{\rm I}}+\frac{1}{2k_{c1}}\tilde{\Omega}_{c}^{\top}R_{d}^{\top}JR_{d}\tilde{\Omega}_{c}+\frac{1}{\delta_{c1}}\boldsymbol{\Upsilon}(\tilde{R}_{c})^{\top}\tilde{\Omega}_{c}\label{eq:VTOL_LyapLc2}
	\end{equation}
	wher{\small{}e
		\[
		e_{c1}^{\top}\underbrace{\left[\begin{array}{cc}
				\underline{\lambda}_{\overline{M}} & -\frac{1}{2\delta_{c1}}\\
				-\frac{1}{2\delta_{c1}} & \frac{1}{2k_{c1}}
			\end{array}\right]}_{Q_{5}}e_{c1}\leq\mathcal{L}_{c1}\leq e_{c1}^{\top}\underbrace{\left[\begin{array}{cc}
				\overline{\lambda}_{\overline{M}} & \frac{1}{2\delta_{c1}}\\
				\frac{1}{2\delta_{c1}} & \frac{1}{2k_{c1}}
			\end{array}\right]}_{Q_{6}}e_{c1}
		\]
	}with $e_{c1}=[\sqrt{||\tilde{R}_{c}||_{{\rm I}}},||\tilde{\Omega}_{c}||]^{\top}$.
	It becomes obvious that $Q_{5}$ and $Q_{6}$ are made positive by
	setting $\delta_{c1}>\sqrt{\frac{k_{c1}}{2\underline{\lambda}_{\overline{M}}}}$.
	From \eqref{eq:VTOL_vex_eq2}, \eqref{eq:VTOL_LyapLc2}, \eqref{eq:VTOL_LyapLc2-1},
	one has
	\begin{align}
		\dot{\mathcal{L}}_{c1}\leq & -e_{c1}^{\top}\underbrace{\left[\begin{array}{cc}
				\frac{k_{c1}c_{c2}}{\delta_{c1}} & \frac{c_{c3}}{2\delta_{c1}}\\
				\frac{c_{c3}}{2\delta_{c1}} & \frac{k_{c2}}{k_{c1}}
			\end{array}\right]}_{D_{c1}}e_{c1}+\frac{c_{c3}}{k_{c1}}||\tilde{\Omega}_{c}||\,||\tilde{b}_{\Omega}||\nonumber \\
		& +\frac{c_{c3}}{\delta_{c1}}||\tilde{b}_{\Omega}||\sqrt{||\tilde{R}_{c}||_{{\rm I}}}\nonumber \\
		\leq & -\underline{\lambda}_{D_{c1}}||e_{c1}||^{2}+c_{r}(||\tilde{\Omega}_{c}||+\sqrt{||\tilde{R}_{c}||_{{\rm I}}})||\tilde{b}_{\Omega}||\label{eq:VTOL_LyapLc2dot}
	\end{align}
	with $c_{r}=\max\{\frac{c_{c3}}{k_{c1}},\frac{c_{c3}}{\delta_{c1}}\}$.
	$D_{c1}$ is positive if $\delta_{c1}>\frac{c_{c3}^{2}}{4c_{c2}k_{c2}}$.
	Let us set $\delta_{c1}>\max\{\sqrt{\frac{k_{c1}}{2\underline{\lambda}_{\overline{M}}}},\frac{c_{c3}^{2}}{4c_{c2}k_{c2}}\}$
	with $\underline{\lambda}_{D_{c1}}$ being the minimum eigenvalue
	of $D_{c1}$. Thereby, one finds
	\begin{align}
		\mathcal{L}_{c1}\leq & -\underline{\lambda}_{D_{c1}}||e_{c1}||^{2}+c_{r}||e_{o1}||\,||e_{c1}||\label{eq:VTOL_LyapLc3dot}
	\end{align}
	where $e_{o1}=[\sqrt{||\tilde{R}_{o}||_{{\rm I}}},||\hat{R}^{\top}\tilde{b}_{\Omega}||]^{\top}$.
	Since $\hat{P}$, $\hat{V}$, $P_{d}$, and $V_{d}$ are bounded,
	$\ddot{\theta}$ and $F$ are bounded and, in turn, $\Im$ is bounded.
	Considering the fact that $||\mathbf{I}_{3}-\tilde{R}_{c}||_{F}=2\sqrt{2}\sqrt{||\tilde{R}_{c}||_{{\rm I}}}$,
	one finds that $||ge_{3}-F||(R^{\top}-R_{d}^{\top})e_{3}\leq4(||ge_{3}||+\sup_{t\geq0}\{||\ddot{P}_{d}||\}+(k_{\theta1}+k_{\theta2})\sqrt{||\tilde{R}_{c}||_{{\rm I}}}\triangleq4\Lambda\,\sqrt{||\tilde{R}_{c}||_{{\rm I}}}$
	indicating that
	\begin{align}
		||ge_{3}-F||(R^{\top}-R_{d}^{\top})e_{3}\leq & 4\Lambda\,\sqrt{||\tilde{R}_{c}||_{{\rm I}}}\label{eq:VTOL_RHO}
	\end{align}
	where $\Lambda$ is an upper bounded positive constant. Based on \eqref{eq:VTOL_Q_E},
	let us introduce the following Lyapunov function candidate $\mathcal{L}_{c2}:\mathbb{R}^{3}\times\mathbb{R}^{3}\rightarrow\mathbb{R}_{+}$:
	\begin{equation}
		\mathcal{L}_{c2}=\frac{1}{2}\mathcal{E}^{\top}\mathcal{E}+\frac{1}{2k_{c3}}\dot{\mathcal{E}}^{\top}\dot{\mathcal{E}}+\frac{1}{\delta_{c2}}\mathcal{E}^{\top}\dot{\mathcal{E}}\label{eq:VTOL_LyapL1C}
	\end{equation}
	such tha{\small{}t}
	\[
	e_{c2}^{\top}\underbrace{\left[\begin{array}{cc}
			\frac{1}{2} & \frac{-1}{2\delta_{c2}}\\
			\frac{-1}{2\delta_{c2}} & \frac{1}{2k_{c3}}
		\end{array}\right]}_{Q_{7}}e_{c2}\leq\mathcal{L}_{c2}\leq e_{c2}^{\top}\underbrace{\left[\begin{array}{cc}
			\frac{1}{2} & \frac{1}{2\delta_{c2}}\\
			\frac{1}{2\delta_{c2}} & \frac{1}{2k_{c3}}
		\end{array}\right]}_{Q_{8}}e_{c2}
	\]
	where $e_{c2}=[||\mathcal{E}||,||\dot{\mathcal{E}}||]^{\top}$. $Q_{7}$
	and $Q_{8}$ can be made positive by selecting $\delta_{c2}>\sqrt{k_{c3}}$.
	Using \eqref{eq:VTOL_Thrust}, \eqref{eq:VTOL_E_dot}, and \eqref{eq:VTOL_RHO},
	one finds
	\begin{align}
		& \dot{\mathcal{L}}_{c2}\leq-e_{c2}^{\top}\underbrace{\left[\begin{array}{cc}
				\frac{k_{c3}}{\delta_{c2}} & \frac{k_{c4}}{2\delta_{c2}}\\
				\frac{k_{c4}}{2\delta_{c2}} & \frac{k_{c4}}{k_{c3}}-\frac{1}{\delta_{c2}}
			\end{array}\right]}_{D_{c2}}e_{c2}\nonumber \\
		& +c_{m}(||\mathcal{E}||+||\dot{\mathcal{E}}||)(||\sum_{i=1}^{n}s_{i}\tilde{y}_{i}||+||\tilde{V}_{o}||+2\sqrt{||\tilde{R}_{c}||_{{\rm I}}})\label{eq:VTOL_LyapL1Cdot}
	\end{align}
	with $c_{p}=\max\{\sup_{t\geq0}||\frac{k_{c3}}{\delta_{c2}}p_{c}+\frac{k_{c4}}{\delta_{c2}}V||,\sup_{t\geq0}||p_{c}+\frac{k_{c4}}{k_{c3}}V||\}$
	and $c_{m}=\max\{\frac{1}{s_{T}},\frac{k_{c4}}{k_{c3}},\frac{k_{c3}}{\delta_{c2}s_{T}},\frac{k_{c4}}{\delta_{c2}},\frac{1}{k_{c2}},\frac{1}{k_{c3}},c_{p},4\Lambda\}$.
	It is worth noting that $\tilde{P}_{o}=\frac{1}{s_{T}}\sum_{i=1}^{n}s_{i}\tilde{y}_{i}-(\tilde{R}_{o}-\mathbf{I}_{3})p_{c}$,
	$\hat{P}-P_{d}-\theta=-\frac{1}{s_{T}}\sum_{i=1}^{n}s_{i}\tilde{y}_{i}+(\tilde{R}_{o}-\mathbf{I}_{3})p_{c}+\mathcal{E}$,
	and $\hat{V}-V_{d}-\dot{\theta}=\tilde{V}_{o}+(\tilde{R}_{o}-\mathbf{I}_{3})V+\dot{\mathcal{E}}$.
	Let us select $\delta_{c2}>\frac{k_{c4}^{2}+4k_{c3}}{4k_{c4}}$ to
	make $D_{c2}$ positive. Also, consider setting $\delta_{c2}>\max\{\sqrt{k_{c3}},\frac{k_{c4}^{2}+4k_{c3}}{4k_{c4}}\}$.
	In view of \eqref{eq:VTOL_LyapLodot_Total}, $\tilde{P}_{o}$ and
	$\tilde{V}_{o}$ are bounded, and based on \eqref{eq:VTOL_LyapLc3dot},
	$\sqrt{||\tilde{R}_{c}||_{{\rm I}}}$ is bounded and, in turn, $\mathcal{L}_{c2}$
	is bounded. Let us use \eqref{eq:VTOL_LyapL1C} and \eqref{eq:VTOL_LyapLc2},
	and define the following Lyapunov function candidate $\mathcal{L}_{cT}:\mathbb{SO}\left(3\right)\times\mathbb{R}^{3}\times\mathbb{R}^{3}\times\mathbb{R}^{3}\rightarrow\mathbb{R}_{+}$:
	\begin{equation}
		\mathcal{L}_{cT}=\mathcal{L}_{c1}+\mathcal{L}_{c2}\label{eq:VTOL_Q_LyapLc-Final}
	\end{equation}
	From \eqref{eq:VTOL_LyapL1Cdot} and \eqref{eq:VTOL_LyapLc2dot},
	one obtains
	\begin{align}
		\dot{\mathcal{L}}_{cT}\leq & -{\bf e}_{c}^{\top}\underbrace{\left[\begin{array}{cc}
				\underline{\lambda}_{D_{c1}} & -c_{m}\\
				-c_{m} & \underline{\lambda}_{D_{c2}}
			\end{array}\right]}_{D_{c}}{\bf e}_{c}+c_{r}||e_{o1}||\,||e_{c1}||\nonumber \\
		& +c_{m}||e_{o2}||\,||e_{c2}||\label{eq:VTOL_Q_LyapLcdot-Final}
	\end{align}
	where ${\bf e}_{c}=[||e_{c1}||,||e_{c2}||]^{\top}$ and $e_{o2}=[||\sum_{i=1}^{n}s_{i}\tilde{y}_{i}||,||\tilde{V}_{o}||]^{\top}$.
	Select $\underline{\lambda}_{D_{c1}}>c_{m}^{2}/\underline{\lambda}_{D_{c2}}$
	to make $D_{c}$ positive. In view of \eqref{eq:VTOL_LyapLo_Total}
	and \eqref{eq:VTOL_Q_LyapLc-Final}, let us consider the following
	Lyapunov function candidate:
	\begin{equation}
		\mathcal{L}_{T}=\mathcal{L}_{oT}+\mathcal{L}_{cT}\label{eq:VTOL_Q_LyapL-Final}
	\end{equation}
	Let $c_{c}=\max\{c_{r},c_{m}\}$. Thereby, from \eqref{eq:VTOL_LyapLodot_Total}
	and \eqref{eq:VTOL_Q_LyapLcdot-Final}, one obtains
	\begin{align}
		\dot{\mathcal{L}}_{T}\leq & -\left[\begin{array}{c}
			||{\bf e}_{o}||\\
			||{\bf e}_{c}||
		\end{array}\right]^{\top}\underbrace{\left[\begin{array}{cc}
				\underline{\lambda}_{D_{o}} & \frac{c_{c}}{2}\\
				\frac{c_{c}}{2} & \underline{\lambda}_{D_{c}}
			\end{array}\right]}_{D}\left[\begin{array}{c}
			||{\bf e}_{o}||\\
			||{\bf e}_{c}||
		\end{array}\right]\label{eq:VTOL_Q_LyapL-Final-dot}
	\end{align}
	where ${\bf e}_{o}=[||e_{o1}||,||e_{o2}||]^{\top}$. $D$ can be made
	positive by setting $\underline{\lambda}_{D_{o}}>\frac{c_{c}^{2}}{4\underline{\lambda}_{D_{c}}}$.
	Hence, let us select $\underline{\lambda}_{D_{o}}>\frac{c_{c}^{2}}{4\underline{\lambda}_{D_{c}}}$.
	Define $\underline{\lambda}_{D}$ as the minimum eigenvalue of $D$.
	By letting $\eta_{Q}=\max\{\overline{\lambda}(Q_{1}),\overline{\lambda}(Q_{2}),\ldots,\overline{\lambda}(Q_{8})\}$,
	one obtains
	\begin{align*}
		\dot{\mathcal{L}}_{T}\leq & -\lambda_{D}(||\tilde{R}_{o}||_{{\rm I}}+||\tilde{b}_{\Omega}||^{2}+||\sum_{i=1}^{n}s_{i}\tilde{y}_{i}||+||\tilde{V}_{o}||^{2})\\
		& -\lambda_{D}(||\tilde{R}_{c}||_{{\rm I}}+||\tilde{\Omega}_{c}||^{2}+|\mathcal{E}||^{2}+||\dot{\mathcal{E}}||^{2})
	\end{align*}
	such that
	\begin{equation}
		\dot{\mathcal{L}}_{T}\leq-(\underline{\lambda}_{D}/\eta_{Q})\mathcal{L}_{T}(t)\label{eq:VTOL_Q_LyapL-dot-Total-1}
	\end{equation}
	and thereby
	\begin{align}
		\mathcal{L}_{T}(t)\leq & \mathcal{L}_{T}(0)\exp(-t\underline{\lambda}_{D}/\eta_{Q}),\hspace{1em}\forall t\geq0\label{eq:VTOL_Q_LyapL-dot-Total}
	\end{align}
	Based on \eqref{eq:VTOL_Q_LyapL-dot-Total}, it becomes apparent that
	$\lim_{t\rightarrow\infty}\tilde{R}_{o}=\lim_{t\rightarrow\infty}\tilde{R}_{c}=\mathbf{I}_{3}$,
	$\lim_{t\rightarrow\infty}||\tilde{\Omega}_{o}||=\lim_{t\rightarrow\infty}||\sum_{i=1}^{n}s_{i}\tilde{y}_{i}||=\lim_{t\rightarrow\infty}||\tilde{V}_{o}||=0$,
	and $\lim_{t\rightarrow\infty}||\tilde{\Omega}_{c}||=\lim_{t\rightarrow\infty}||\mathcal{E}||=\lim_{t\rightarrow\infty}||\dot{\mathcal{E}}||=0$.
	According to the definition of $\ddot{\theta}$ in \eqref{eq:VTOL_Thrust}
	and the convergence of $\mathcal{L}_{T}(t)$ to the origin, $\ddot{\theta}\rightarrow-k_{\theta1}\tanh(\theta)-k_{\theta2}\tanh(\dot{\theta})$
	as $\mathcal{E},\dot{\mathcal{E}}\rightarrow0$ which shows that $||\tanh(\theta)||$
	and $||\tanh(\dot{\theta})||$ are strictly decreasing with $||\tanh(\theta)||\rightarrow0$
	and $||\tanh(\dot{\theta})||\rightarrow0$, and therefore $\lim_{t\rightarrow\infty}||\theta||=\lim_{t\rightarrow\infty}||\dot{\theta}||=0$.
	Thus, the closed-loop error signals of the observer-based controller
	design are uniformly almost globally exponentially stable. This completes
	the proof of Theorem \ref{thm:Theorem2}.\end{proof}

\section{Implementation Steps \label{sec:VTOL_Implementation}}

In this subsection, the proposed direct VTOL-UAV observer-based controller
is implemented in its discrete form. Let $\Delta t$ be a small sample
time step. By following the steps below, one can seamlessly implement
the novel observer-based controller. 

\textbf{Step 1.} Set $\hat{b}_{\Omega|0},\hat{P}_{0},\hat{V}_{0},\theta_{0},\dot{\theta}_{0}\in\mathbb{R}^{3}$,
$\hat{R}_{0}\in\mathbb{SO}(3)$ with $\hat{X}_{0}=\left[\begin{array}{ccc}
	\hat{R}_{0}^{\top} & \hat{P}_{0} & \hat{V}_{0}\\
	0_{1\times3} & 1 & 0\\
	0_{1\times3} & 0 & 1
\end{array}\right]$, and $k=1$.\vspace{0.2cm}

\textbf{Step 2.} (Thrust evaluation) Calculate the auxiliary variable
$\ddot{\theta}_{k}=-k_{\theta1}\tanh(\theta_{k-1})-k_{\theta2}\tanh(\dot{\theta}_{k-1})+k_{c3}(\hat{P}_{k}-P_{d|k}-\theta_{k-1})+k_{c4}(\hat{V}_{k}-V_{d|k}-\dot{\theta}_{k-1})$
where $\dot{\theta}_{k}=\dot{\theta}_{k-1}+\Delta t\ddot{\theta}_{k}$
and $\theta_{k}=\theta_{k-1}+\Delta t\dot{\theta}_{k}$. Next, evaluate
the intermediary control input and the thrust as in \eqref{eq:VTOL_Thrust}
\begin{align*}
	F_{k} & =\ddot{P}_{d}-k_{\theta1}\tanh(\theta_{k})-k_{\theta2}\tanh(\dot{\theta}_{k})=[f_{1},f_{2},f_{3}]^{\top}\\
	\Im_{k} & =m||ge_{3}-F_{k}||
\end{align*}

\textbf{Step 3.} (Desired unit-quaternion) The desired unit-quaternion
vector can be evaluated as in \eqref{eq:VTOL_Q_Qd} by
\[
q_{d0|k}=\sqrt{\frac{m}{2\Im_{k}}(g-f_{3})+\frac{1}{2}},\hspace{1em}q_{d|k}=\left[\begin{array}{c}
	\frac{m}{2\Im_{k}q_{d0}}f_{2}\\
	-\frac{m}{2\Im_{k}q_{d0}}f_{1}\\
	0
\end{array}\right]
\]
with $F_{k}=[f_{1},f_{2},f_{3}]^{\top}\in\mathbb{R}^{3}$, $Q_{d|k}=[q_{d0|k},q_{d|k}^{\top}]^{\top}\in\mathbb{S}^{3}$,
and 
\[
R_{d|k}=(q_{d0|k}^{2}-||q_{d|k}||^{2})\mathbf{I}_{3}+2q_{d|k}q_{d|k}^{\top}-2q_{d0|k}[q_{d|k}]_{\times}\in\mathbb{SO}\left(3\right)
\]

\textbf{Step 4.} (Direct sensor measurements and correction factors)
Collect the following set of measurements for the observer design:
\[
\begin{cases}
	\boldsymbol{\Upsilon}(\tilde{R}_{o}M) & =\sum_{i=1}^{n}\frac{s_{i}}{2}\left((p_{i}-p_{c})\times\hat{R}_{k}^{\top}y_{i}\right)\\
	\sum_{i=1}^{n}s_{i}\tilde{y}_{i} & =\sum_{i=1}^{n}s_{i}(\hat{P}_{k}+\hat{R}_{k}^{\top}y_{i}-p_{i})
\end{cases}
\]
The correction factors can be calculated using
\[
\begin{cases}
	w_{\Omega} & =k_{o1}\sum_{i=1}^{n}\frac{s_{i}}{2}\left((p_{i}-p_{c})\times\hat{R}_{k}^{\top}y_{i}\right)\\
	w_{V} & =k_{o2}\sum_{i=1}^{n}s_{i}\tilde{y}_{i}-\frac{1}{s_{T}}[w_{\Omega}]_{\times}(\sum_{i=1}^{n}s_{i}\tilde{y}_{i}+s_{T}p_{c})\\
	w_{a} & =-ge_{3}+k_{o3}\sum_{i=1}^{n}s_{i}\tilde{y}_{i}
\end{cases}
\]

\textbf{Step 5.} (Prediction) The estimates of attitude, position,
and linear velocity are defined by
\begin{align*}
	\hat{U}_{k} & =\left[\begin{array}{ccc}
		[\Omega_{m}-\hat{b}_{\Omega}\text{\ensuremath{]_{\times}}} & 0_{3\times1} & -\frac{\Im_{k}}{m}e_{3}\\
		0_{1\times3} & 0 & 0\\
		0_{1\times3} & 1 & 0
	\end{array}\right]\in\mathcal{U}_{\mathcal{M}}\\
	\hat{X}_{k|k-1} & =\hat{X}_{k-1}\exp(\hat{U}_{k}\Delta t)
\end{align*}
$\exp(\cdot)$ stands for exponential of a matrix.\vspace{0.2cm}

\textbf{Step 6.} (Correction) $W=\left[\begin{array}{ccc}
	\left[w_{\Omega}\right]_{\times} & w_{V} & w_{a}\\
	0_{1\times3} & 0 & 0\\
	0_{1\times3} & 1 & 0
\end{array}\right]\in\mathcal{U}_{\mathcal{M}}$ and 
\begin{align*}
	\hat{X}_{k} & =\exp(-W\Delta t)\hat{X}_{k|k-1}
\end{align*}
with
\begin{align*}
	\hat{P}_{k} & =\hat{X}_{k}(1:3,4)\\
	\hat{V}_{k} & =\hat{X}_{k}(1:3,5)\\
	\hat{R}_{k} & =\hat{X}_{k}(1:3,1:3)^{\top}
\end{align*}

\textbf{Step 7.} Derivatives of the intermediary control inputs can
be calculated as in \eqref{eq:VTOL_Fdot}
\begin{align*}
	\dot{F}_{k} & =P_{d}^{(3)}-k_{\theta1}H\dot{\theta}_{k}-k_{\theta2}\dot{H}\ddot{\theta}_{k}\\
	\ddot{F}_{k} & =P_{d}^{(4)}-k_{\theta1}Z-k_{\theta2}\dot{Z}
\end{align*}
Next, $\Xi(F)$ can be calculated, see \eqref{eq:VTOL_Q_Lambda},
as well as its derivative $\dot{\Xi}(F)$.\vspace{0.2cm}

\textbf{Step 8.} The desired angular velocity and its derivative can
be calculated as in \eqref{eq:VTOL_Q_Omd} and \eqref{eq:VTOL_Q_Omd_dot-1},
respectively
\[
\Omega_{d|k}=\Xi(F_{k})\dot{F}_{k},\hspace{1em}\dot{\Omega}_{d|k}=\dot{\Xi}(F_{k})\dot{F}_{k}+\Xi(F_{k})\ddot{F}_{k}
\]

\textbf{Step 9.} (Torque input) The rotational torque is defined by
\[
\begin{cases}
	\boldsymbol{\Upsilon}(\tilde{R}_{c}M) & =\sum_{i=1}^{n}\frac{s_{i}}{2}\left((p_{i}-p_{c})\times R_{d}^{\top}y_{i}\right)\\
	w_{c} & =k_{c1}R_{d|k}\sum_{i=1}^{n}\frac{s_{i}}{2}\left((p_{i}-p_{c})\times R_{d|k}^{\top}y_{i}\right)\\
	& \hspace{1em}+k_{c2}(\Omega_{d|k}-\Omega_{m}+\hat{b}_{\Omega})\\
	\mathcal{T} & =w_{c}+J\dot{\Omega}_{d|k}-\left[J(\Omega_{m}-\hat{b}_{\Omega})\right]_{\times}\Omega_{d|k}
\end{cases}
\]

\textbf{Step 10.} (Gyro bias estimate) The estimate of gyro bias can
be calculated using
\[
\hat{b}_{\Omega|k+1}=\hat{b}_{\Omega|k}+\Delta t\gamma_{o}\hat{R}_{k}\sum_{i=1}^{n}\frac{s_{i}}{2}\left((p_{i}-p_{c})\times\hat{R}_{k}^{\top}y_{i}\right)
\]

\textbf{Step 11.} Set $k=k+1$ and go to \textbf{Step 2}.

\section{Numerical Results\label{sec:SE3_Simulations}}

In this Section, the proposed direct observer-based controller for
VTOL-UAV is evaluated considering (i) a VA-INS unit at a low sampling
rate of 1000 Hz, (ii) large error initialization, as well as (iii)
unknown uncertain components inherent in the measurements. Let us
set the initial values of the VTOL-UAV motion components as follows:
\[
R_{0}=\left[\begin{array}{ccc}
	0.5763 & -0.7638 & 0.2907\\
	0.8147 & 0.5085 & -0.2789\\
	0.0652 & 0.3976 & 0.9153
\end{array}\right]\in\mathbb{SO}(3)
\]
initial position $P_{0}=[-1,-1,0]^{\top}$, initial angular velocity
$\Omega_{0}=[0,0,0]^{\top}$, and initial linear velocity $V_{0}=[1,1,0]^{\top}$.
Let the UAV mass be $m=3\,\text{kg}$ and the inertia be $J={\rm diag}(0.15,0.23,0.16)\,\text{kg}.\text{m}^{2}$.
Define the desired trajectory as 
\[
P_{d}=6[\cos(0.19t),\sin(0.2t)\cos(0.2t),\frac{1}{6}(3.5+0.15t)]^{\top}\,\text{m}
\]
with a total travel time of 50 seconds. Let us set the initial estimates
of the VTOL-UAV motion components as follows: $\hat{R}_{0}=\mathbf{I}_{3}$,
$\hat{P}_{0}=\hat{V}_{0}=[0,0,0]^{\top}$, and $\hat{b}_{\Omega}(0)=[0,0,0]^{\top}$.
In addition, set $\hat{\theta}_{0}=\dot{\hat{\theta}}_{0}=[0,0,0]^{\top}$,
and consider a set of five randomly distributed non-collinear features
to satisfy Assumption \ref{Assum:VTOL_1Landmark}. Let the measurements
in \eqref{eq:VTOL_VRP} be corrupted with constant bias and normally
distributed random noise with a zero mean and a standard deviation
of 0.07 = $\mathcal{N}(0,0.07)$. Let the design parameters be selected
as $\gamma_{o}=0.7$, $k_{o1}=11$, $k_{o2}=10$, $k_{o3}=4$, $k_{\theta1}=1.2$,
$k_{\theta2}=1.2$, $k_{c1}=1$, $k_{c2}=4$, $k_{c3}=4$, and $k_{c4}=2$.

\begin{figure}[H]
	\centering{}\includegraphics[scale=0.32]{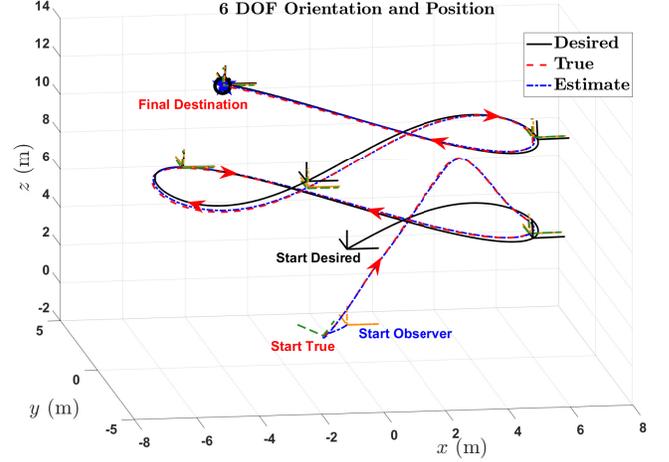}\caption{Output performance of the 6 DoF VTOL-UAV observer-based controller.}
	\label{fig:ObsvCont1}
\end{figure}

\begin{figure}[H]
	\centering{}\includegraphics[scale=0.35]{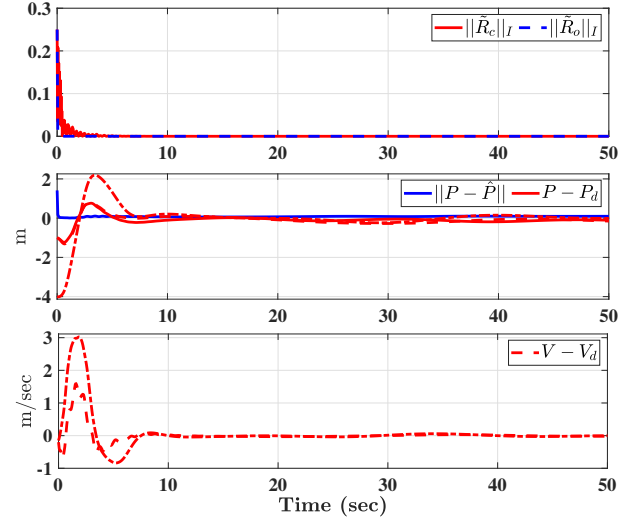}\caption{Estimation and control error trajectories: attitude, angular velocity,
		position, linear velocity.}
	\label{fig:ObsvCont2}
\end{figure}

Fig. \ref{fig:ObsvCont1} depicts the output performance of the novel
observer-based controller comparing the trajectories of the desired,
true, and estimated position of the VTOL-UAV plotted in black solid
line, red dashed line, and blue center line, respectively. The desired,
true, and estimated VTOL-UAV orientations represented by roll, yaw,
and pitch are contrasted using black solid line, green dashed line,
and orange center-line, respectively. Fig. \ref{fig:ObsvCont1} shows
robust tracking performance from a large initialization error to the
desired destination. Furthermore, the convergence of the error trajectories
is demonstrated in Fig. \ref{fig:ObsvCont2}, where the error between
the true and the estimated response plotted in blue is compared to
the error between the true and the desired response plotted in red.
Fig. \ref{fig:ObsvCont2} reveals fast adaptation and accurate tracking
of the observer performance with respect to the true trajectory. Likewise,
the fast convergence of the error signals from large values to the
neighborhood of the origin in Fig. \ref{fig:ObsvCont2}, illustrates
the robust tracking control of the VTOL-UAV to the desired trajectory
enabled by the novel observer-based controller. Fig. \ref{fig:ControlInput}
illustrates bounded control input signals of the rotational torque
and thrust.

\begin{figure}
	\centering{}\includegraphics[scale=0.25]{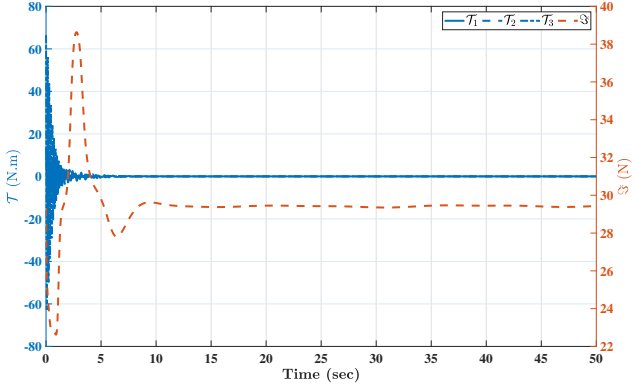}\caption{VTOL-UAV control input: rotational torque and thrust.}
	\label{fig:ControlInput}
\end{figure}

\section{Conclusion \label{sec:SE3_Conclusion}}

This work tackles the estimation and control problem of a VTOL-UAV
traveling in a three-dimensional space (3D) utilizing exclusively
a typical 6-axis IMU (gyroscope and accelerometer) and feature measurements.
This paper introduces a novel direct nonlinear observer that mimics
the true nonlinearity of the VTOL-UAV motion dynamics. The proposed
observer accurately estimates VTOL-UAV motion components, namely attitude,
position, and linear velocity. The combination of the proposed observer
with the novel control laws creates a comprehensive autonomous module,
where the closed-loop error signals of both the observer and the observer-based
controller are guaranteed to be exponentially stable starting from
almost any initial condition. Testing of the proposed approach revealed
strong tracking capabilities for the 3D motion. 

\section*{Acknowledgment}

The authors would like to thank \textbf{Maria Shaposhnikova} for proofreading
the article.

\subsection*{Appendix\label{subsec:Appendix-A}}
\begin{center}
	\textbf{Equivalent Design in Quaternion Form}
	\par\end{center}

\noindent Define $Q=[q_{0},q^{\top}]^{\top}\in\mathbb{S}^{3}$ to
be the VTOL-UAV true unit-quaternion where $q\in\mathbb{R}^{3}$ and
$q_{0}\in\mathbb{R}$ with $\mathbb{S}^{3}=\{\left.Q\in\mathbb{R}^{4}\right|||Q||=1\}$
\cite{hashim2019AtiitudeSurvey,shuster1993survey}. Let $\hat{Q}=[\hat{q}_{0},\hat{q}^{\top}]^{\top}\in\mathbb{S}^{3}$
and $Q_{d}=[q_{d0},q_{d}^{\top}]^{\top}\in\mathbb{S}^{3}$ be the
estimated and the desired unit-quaternion, respectively. Consider
the mapping $\hat{\mathcal{R}}:\mathbb{S}^{3}\rightarrow\mathbb{SO}\left(3\right)$
and $\mathcal{R}_{d}:\mathbb{S}^{3}\rightarrow\mathbb{SO}\left(3\right)$
as \cite{hashim2019AtiitudeSurvey,shuster1993survey}
\[
\begin{cases}
	\hat{\mathcal{R}} & =(\hat{q}_{0}^{2}-||\hat{q}||^{2})\mathbf{I}_{3}+2\hat{q}\hat{q}^{\top}-2\hat{q}_{0}[\hat{q}]_{\times}\\
	\mathcal{R}_{d} & =(q_{d0}^{2}-||q_{d}||^{2})\mathbf{I}_{3}+2q_{d}q_{d}^{\top}-2q_{d0}[q_{d}]_{\times}
\end{cases}
\]
where $\hat{\mathcal{R}},\mathcal{R}_{d}\in\mathbb{SO}\left(3\right)$.
Recall \eqref{eq:VTOL_Measurements} and consider the following direct
measurement set-up:
\[
\begin{cases}
	\boldsymbol{\Upsilon}_{o} & =\mathbf{vex}(\boldsymbol{\mathcal{P}}_{a}(\sum_{i=1}^{n}s_{i}\hat{\mathcal{R}}^{\top}y_{i}(p_{i}-p_{c})^{\top}))\\
	\sum_{i=1}^{n}s_{i}\tilde{y}_{i} & =\sum_{i=1}^{n}s_{i}(\hat{P}+\hat{\mathcal{R}}^{\top}y_{i}-p_{i})
\end{cases}
\]
The unit-quaternion representation equivalent to \eqref{eq:VTOL_ObsvCompact}
is given by
\[
\begin{cases}
	\dot{\hat{b}}_{\Omega} & =\gamma_{o}\hat{\mathcal{R}}\boldsymbol{\Upsilon}_{o}\\
	w_{\Omega} & =k_{o1}\boldsymbol{\Upsilon}_{o}\\
	w_{V} & =k_{o2}\sum_{i=1}^{n}s_{i}\tilde{y}_{i}-\frac{1}{s_{T}}[w_{\Omega}]_{\times}(\sum_{i=1}^{n}s_{i}\tilde{y}_{i}+s_{T}p_{c})\\
	w_{a} & =-ge_{3}+k_{o3}\sum_{i=1}^{n}s_{i}\tilde{y}_{i}
\end{cases}
\]
\[
\begin{cases}
	Y & =\left[\begin{array}{cc}
		0 & -\hat{\Omega}^{\top}\\
		\hat{\Omega} & -[\hat{\Omega}]_{\times}
	\end{array}\right],\hspace{1em}Z=\left[\begin{array}{cc}
		0 & -w_{\Omega}^{\top}\\
		w_{\Omega} & [w_{\Omega}]_{\times}
	\end{array}\right]\\
	\dot{\hat{Q}} & =\frac{1}{2}(Y-Z)\hat{Q}\\
	\dot{\hat{P}} & =\hat{V}-[w_{\Omega}]_{\times}\hat{P}-w_{V}\\
	\dot{\hat{V}} & =-\frac{\Im}{m}\hat{\mathcal{R}}^{\top}e_{3}-[w_{\Omega}]_{\times}\hat{V}-w_{a}
\end{cases}
\]
The control laws in \eqref{eq:VTOL_Tau}-\eqref{eq:VTOL_Thrust} in
terms of quaternion form are given below:
\begin{align*}
	& \begin{cases}
		\boldsymbol{\Upsilon}_{c} & =\mathbf{vex}(\boldsymbol{\mathcal{P}}_{a}(\mathcal{R}_{d}^{\top}\sum_{i=1}^{n}s_{i}y_{i}(p_{i}-p_{c})^{\top}))\\
		w_{c} & =k_{c1}\mathcal{R}_{d}\boldsymbol{\Upsilon}_{c}+k_{c2}(\Omega_{d}-\Omega_{m}+\hat{b}_{\Omega})\\
		\mathcal{T} & =w_{c}+J\dot{\Omega}_{d}-\left[J(\Omega_{m}-\hat{b}_{\Omega})\right]_{\times}\Omega_{d}
	\end{cases}\\
	& \begin{cases}
		\ddot{\theta} & =-k_{\theta1}\tanh(\theta)-k_{\theta2}\tanh(\dot{\theta})\\
		& \hspace{1em}+k_{c3}(\hat{P}-P_{d}-\theta)+k_{c4}(\hat{V}-V_{d}-\dot{\theta})\\
		F & =\ddot{P}_{d}-k_{\theta1}\tanh(\theta)-k_{\theta2}\tanh(\dot{\theta})\\
		\Im & =m||ge_{3}-F||
	\end{cases}
\end{align*}

\balance

\bibliographystyle{IEEEtran}
\bibliography{bib_VTOL_ObsvCont}

\end{document}